\documentclass[twoside,11pt]{article}

 \usepackage[abbrvbib, preprint]{jmlr2e}

\usepackage[utf8]{inputenc} 
\usepackage[T1]{fontenc}    
\usepackage{hyperref}       
\usepackage{url}            
\usepackage{booktabs}       
\usepackage{amsfonts}       
\usepackage{nicefrac}       
\usepackage{microtype}      
\usepackage{amsmath}
\usepackage{todonotes}
\usepackage{subcaption}
\usepackage[ruled]{algorithm2e}
\usepackage{wrapfig}

\graphicspath{{../images/}}

\newcommand{\mB}{\mathcal{B} } 
\newcommand{\mA}{\mathcal{A} } 
\newcommand{\mT}{\mathcal{T} } 
 
\newcommand{\mDA}{\mathcal{D}(\mathcal{A}) } 
\newcommand{\mE}{\mathbb{E}} 
 
\newcommand{\mX}{X} 
\newcommand{\mR}{\mathbb{R}} 
\newcommand{\mRange}{\mathcal{R}}

\newcommand{\mH}{\mathcal{H} } 
 
\newcommand{\mID}{{\rm id}}

\newtheorem{assumption}{Assumption}

\jmlrheading{00}{2020}{1-34}{4/00}{10/00}{meila00a}{Tao Bian and Zhong-Ping Jiang}

\ShortHeadings{Temporal-Differential Learning in Continuous Environments}{Bian and Jiang}
\firstpageno{1}

\begin{document}

\title{Temporal-Differential Learning in Continuous Environments}

\author{\name Tao Bian \email tbian@nyu.edu\\
       \addr Bank of America\\
       One Bryant Park \\
    New York, NY 10036, USA
       \AND
       \name Zhong-Ping Jiang  \email zjiang@nyu.edu \\
       \addr Department of Electrical and Computer Engineering \\
   New York University \\
   6 Metrotech Center\\
 Brooklyn, NY 11201, USA}

\editor{}

\maketitle

\begin{abstract}
In this paper, a new reinforcement learning (RL) method known as the method of temporal differential is introduced. 
Compared to the traditional temporal-difference learning method, it plays a crucial role in developing novel RL techniques for continuous environments.
    In particular, the continuous-time least squares policy evaluation (CT-LSPE) and the continuous-time temporal-differential (CT-TD) learning methods are developed. 
    Both theoretical and empirical evidences are provided to demonstrate the effectiveness of the proposed temporal-differential learning methodology.
\end{abstract}


\section{Introduction}
Over the past decade, reinforcement learning (RL) has quickly become one of the most prominent technologies in the field of artificial intelligence (AI).
The popularity of RL is mainly driven by its recent success  in various application fields \citep{Barto2017,Silver2017,Wang2018,Kolm2019}. 
Despite the rapid growth of RL in recent years, a common feature in most existing RL techniques, or broadly speaking most AI technologies, is that the learning task is performed in discrete learning environments \citep[Section 2.3.2]{Russell2010}, that is, either the time space or the state-action space or both are associated with the discrete topology. 
Although most practical physical environments are continuous in nature, in the sense that both the time space and the state-action spaces are continuous, RL and AI problems in continuous learning environments have rarely been touched upun in the past literature.
In fact, the ability of learning in continuous learning environments has long been visioned as an important component of a futuristic AI system \citep[Section 7]{Russell1997}.
Admittedly, building AI and RL agents in continuous environments is a challenging task.
One challenge is that most processes evolving in continuous environments are described by ordinary differential equations (ODEs), stochastic differential equations (SDEs), or other complex mathematical models.
These models are still not fully-understood from the AI and RL perspective.  
This raises the question on how to develop a generic learning framework applicable to all these models.
In addition, the state-action spaces encountered in practical RL problems often exhibit complex shapes and structures.
For example, the presence of singularity points in robot workspace \citep{Merlet2006} is a notorious issue in robot control and learning.
These irregularities create a serious challenge for the convergence and robustness analysis
of learning algorithms.
Besides its potential value in practical applications, understanding the continuous limits of RL and AI algorithms also provides us a unique insight into the dynamics and theoretical nature of these algorithms.
In fact, dynamical systems and differential equation theories have gained great popularity recently in understanding ML, especially neural network (NN) models \citep{E2017a,Chaudhari2018,Chen2018,Ruthotto2019}.
In addition, the ODE method has already been used in stochastic approximation \citep{Ljung1977} to study the convergence of RL algorithms \citep{Kushner2003}.
Hence, in order to expand the application of  RL to a broader range of  problems, and to enrich our understanding of RL theory, it is necessary to investigate RL methods in general continuous learning environments. 

The purpose of this paper is to extend the classical temporal difference (TD) error \citep{Sutton1988}  to the setting of continuous environments. 
To overcome the shortcomings of traditional RL in dealing with problems arising from continuous environments, we will introduce
a new concept under the name of {\em temporal-differential} error.
Similar to the TD in discrete environments, the temporal-differential error plays a key role in developing new RL methods in continuous environments.
In particular,  we introduce two continuous RL algorithms, namely,  the continuous-time least squares policy evaluation (CT-LSPE) (Algorithm \ref{alg:temporal-differential}) and the continuous-time temporal-differential (CT-TD) learning (Algorithm \ref{alg:temporal-differential-infty}).
These algorithms are extensions of the LSPE \citep{Nedic2003} and the TD learning \citep[Chapter 6]{Sutton2018}, respectively, to continuous environments.
As a continuous RL methodology, our algorithms also possess several unique features that distinguish them from most traditional RL methods.
First, instead of using NN approximation, kernel approximation is used in our design.
The linkage between the kernel function and the Hilbert space theory allows our algorithms to achieve convergence over the entire state space.
Second, systems and control theory will be applied to show the stability and robustness of the proposed learning algorithms.
Finite-time error bounds are also obtained, without assuming the stationarity of the training data as in the past literature \citep{Antos2008,Munos2008}.

\section{Related Works}
The study on RL in continuous environments can be traced back to the concept of advantage updating \citep{Baird-III1993,Harmon1995}. 
Shortly after Baird's advantage updating, the continuous-time TD learning and actor-critic algorithms were  developed by \citet{Doya1996} and \citet{Doya2000}.
In addition,  \citet{Munos2000} developed a continuous RL technique to solve viscosity solutions in order to tackle the case where the classical optimal solution does not exist.
Continuous RL has also been applied in different fields including game theory \citep{Borgers1997} and neurobiology \citep{Fremaux2013}.
A common feature of these early results is that the discrete-time sampling is performed to convert the continuous-time problem into the discrete-time setting.
More recently, the research on deep RL opens possibilities for developing practical learning algorithms in continuous state-action spaces \citep{Hasselt2007,Hasselt2012,Lillicrap2016,Duan2016,Recht2019}.
Indeed, computer experiments from balancing inverted pendulums to learning the locomotion have shown very promising results.

Besides the above collective efforts, a learning-based control design methodology, usually under the name of adaptive (or approximate) dynamic programming (ADP) \citep{Jiang2012a,Vrabie2013,Bian2016b,Kiumarsi2018}, has been developed to solve optimal control problems for continuous-time dynamical systems.
Instead of discretizing a continuous-time system into a discrete-time system, ADP algorithms solve directly the optimal control problem by  utilizing the continuous-time data flow generated by dynamical systems. 
Exploiting explicitly systems and control techniques,
the  stability and optimality of the controlled system can be guaranteed.
Finally, to perform the learning task in more complex environments, robust and decentralized extensions of ADP have also been proposed by \citet{Jiang2017,Bian2019a}.




\section{Problem Formulation}\label{sec:problem-formulation}
Throughout this paper, $\mathbb{R}$ denotes the set of reals.
$I$ denotes the identity matrix with appropriate dimension.  
$|\cdot|$ denotes the Euclidean norm for vectors, or the induced matrix norm for matrices. 
$\mID$ denotes the identity mapping. 
Let $(\Omega, \mathcal{F}, \mathbb{P})$ be a probability space and $\{x_t\}_{t\geq0}$ be a c{\`a}dl{\`a}g and time-homogeneous strong Markov process with state space $\mX$.
Here $\mX$ is a locally compact Hausdorff space with a countable base, equipped with the Borel $\sigma$-algebra $\mB$.
Denote by $\mE_x$ the expectation conditionally on the initial state $x$ and by $\mE_\mu$ the expectation with respect to $\mu$.
The Hilbert space $L^2_\mu$ and the norm $\|\cdot\|_\mu$ are induced from the inner product $\langle\cdot,\cdot\rangle_\mu$ defined via $\langle f,g\rangle_\mu =  \mE_\mu\left[ fg\right]$ for all $f, g\in L^2_\mu$. 
Denote by $\mA$ the infinitesimal generator associated with $\{x_t\}_{t\geq0}$.
The domain of $\mA$ is written as $\mDA\subseteq L^2_\mu$.
For more details on the notations used in this paper, see Appendix \ref{appendix:markov} and references therein. 

Our goal in this paper is to estimate the following value function $V^*$ that depends on the future path of $\{x_t\}_{t\geq0}$: 
\begin{align}
    V^*(x)&=\mE_{x}\left[\int_0^{\infty}e^{-\gamma t}r(x_t)dt\right],\quad \forall x\in\mX,\label{equ:cost-discount}
\end{align}
where 
$r\in L^2_\mu$ is the running-reward/cost function,
and $\gamma>0$ is a scalar representing the discounting effect.
To be mathematically concrete, we assume $\mE_x[r(x_t)]$ is measurable and 
locally essentially bounded.




To see why the classical discrete-time TD learning is no longer a good choice to solve $V^*$ here, we write the discrete-time TD error for the discretized process (with sampling period $\Delta t$) as 
\begin{align*}
    e^{-\gamma \Delta t}V_k(x_{\Delta t})-V_k(x)+\int_0^{\Delta t}e^{-\gamma s}r(x_s)ds,\quad \forall x\in\mX.
\end{align*}
Obviously, when $\Delta t$ is small, the discounting term $e^{-\gamma \Delta t}$ is close to $1$, resulting in a poor convergence performance (see \citet{Tsitsiklis1997} for convergence analysis).
If $\Delta t$ is too large, the updating frequency is reduced, which could also compromise the on-line learning performance.
In light of the above difficuities, it is desirable to solve $V^*$ without discretizing the performance index and the underlying Markov process in the first place.
To proceed, note from \citet[Chapter 1]{Ito1960} that $V^*$ satisfies the following linear functional equation:
\begin{align}
    0=\mT V^*+r,\label{equ:pe}
\end{align}
where $\mT=\mA -\gamma $.
The converse statement is also true by Dynkin's formula \citep[Chapter 1]{Kushner1967}.

In this paper, we will introduce two on-line learning methods to solve \eqref{equ:pe}.

\section{Value Approximation  in Continuous Time}\label{sec:approximate-pe}

In this section, we consider the following approximated equation:
\begin{align}
    \partial_tV_t=\Pi (\mT V_t+r),\label{equ:approximated-pe}
\end{align}
where $\Pi$ is defined by the following integral equation:
\begin{align*}
    \Pi V(x) = \int_{\mX} V(z)K(z,x) d\mu(z),\quad \forall V\in L^2_\mu,\ x\in\mX.
\end{align*}
Here $K$ is a square-integrable continuous  symmetric positive-definite kernel.
Corresponding to each $K$, we can uniquely define a reproducing kernel Hilbert space (RKHS) $\mH\subseteq L^2_\mu$ \citep[Section 2]{Aronszajn1950}.
In addition, $\Pi$ is an orthogonal projection from $L_\mu^2$ to $\mH$ \citep[Section 2]{Aronszajn1950}. 
To fit into our learning framework, the reproducing kernel $K$ is designed to ensure $\mH$ is a subspace of $\mDA$.
This requirement can be satisfied for many practical processes by selecting a sufficiently smooth $K$; see Section \ref{sec:td} for details.
Finally, we assume the initial value function $V_0\in \mH$.

From the perspective of systems and control theory, \eqref{equ:approximated-pe} can be viewed as a reduced-order model \citep{Scarciotti2017a} of  the nominal system
\begin{align}
\partial_t v_t = \mT v_t+r.\label{equ:pe-full}
\end{align}
As a result, we can rewrite \eqref{equ:approximated-pe} in the form of system \eqref{equ:pe-full} with an approximation error $(\Pi-\mID)(\mT V_t+r)$ added to its right-hand side. 

The following theorem shows that $V_t$ converges into a neighbourhood of $V^*$ exponentially.
\begin{theorem}\label{thm:approximated-pe}
$V_t$ is well-defined and is exponentially stable (at $V_\infty$) in $\mH$.
In addition,  
\begin{align*}
    \left\| V_t-V^*\right\|_\mu & \leq e^{-\gamma t}\left\| V_0- V_\infty\right\|_\mu + \frac{1}{\gamma}\left\|\mT V_\infty+r\right\|_{\mu},\\
    \left\| V_t-V^*\right\|_\mu
      &   \leq e^{-\gamma t} \left\|V_0-V^*\right\|_\mu+\frac{1-e^{-\gamma t}}{\gamma}\sup_{s\in[0,t]}\left\|(\Pi-\mID)(\mT V_s+r)\right\|_{\mu}.
\end{align*}
In particular, $\left\| V_\infty-V^*\right\|_\mu\leq  \gamma^{-1}\left\|\mT V_\infty+r\right\|_{\mu}$.
\end{theorem}
See Appendix \ref{appendix:proof-approximated-pe} for the proof of Theorem \ref{thm:approximated-pe}.

The finite-time error bounds in Theorem \ref{thm:approximated-pe} are composed with two parts.
The first part represents the error induced from the initial guess $ V_0$.
The second part is due to the approximation error introduced by  the projection $\Pi$.
If we can choose $\mH$ properly, such that for any $f\in\mH$, $\mA f\in\mH$, then one can show that
\begin{align*}
    \left\| V_t-V^*\right\|_\mu
     \leq e^{-\gamma t} \left\|V_0-V^*\right\|_\mu+\frac{1-e^{-\gamma t}}{\gamma}\|\Pi r-r\|_\mu.
\end{align*}
The above inequality implies that in this case the error bound is purely controlled by the initial guess $V_0$ and the ``observation error'' in the  running-reward/cost $r$. 

Note that the first error bound in Theorem \ref{thm:approximated-pe} is tighter for large $t$, and the second error bound is tighter for small $t$.
From the perspective of robust control theory, the second error bound quantifies the impact of $(\Pi-\mID)(\mT V_t+r)$ on the approximation error between $V_t$ and $V^*$.
In addition, the discounting factor $\gamma$ is linked to the robustness of \eqref{equ:approximated-pe},
in the sense that  a larger $\gamma$ results in a smaller error bound and a faster convergence rate.

Finally, it is worth pointing out that the results in Theorem \ref{thm:approximated-pe} can also be extended to evaluate the ergodic cost \citep{Arapostathis2012}.
In this case,  there is no discounting term in \eqref{equ:cost-discount},  and we have $\mT=\mA$ in \eqref{equ:pe} and \eqref{equ:approximated-pe}.
Note from \citet[Proposition 2.2]{Bhattacharya1982} that as long as $\{x_t\}_{t\geq0}$ is ergodic,  $\mA$ has a simple zero eigenvalue corresponding to the constant eigenfunction.
Hence, by excluding the constant function from $\mH$, we can ensure there still exists a $\gamma>0$, such that  $\langle V, \Pi\mT V\rangle_\mu\leq -\gamma \|V\|_\mu$. 
As a result, the two inequalities in Theorem \ref{thm:approximated-pe} still hold.

\section{Model-free On-line Learning}\label{sec:td}

The results presented in the last section provide an efficient way to approximate $V^*$ through the linear functional equation \eqref{equ:approximated-pe}.
Unfortunately, in order to solve \eqref{equ:approximated-pe} directly,  we must have access to the knowledge of $\mu$, $\mA$, and $r$.
This is not an easy task in practice, since $\mu$, $\mA$, and $r$ are associated with the mathematical model of the learning environment, which is usually not known.



In this section, we propose two on-line learning algorithms to estimate $V^*$ using the data  observed directly from the environment.
The knowledge of $\mu$, $\mA$, and $r$ is not required in our learning algorithm design.
Throughout this section, we consider the following realization of $K$:
\begin{align*}
    K(z,x)=\phi^T(z)\left(\mE_\mu\left[\phi\phi^T\right]\right)^{-1}\phi(x),
\end{align*}
where $\phi=[\phi_1\ \phi_2\ \cdots \phi_N]^T$, and $\{\phi_j\}_{j=1}^N$ are $N$ independent functions in $\mDA$.
The independence of $\{\phi_j\}_{j=1}^N$ guarantees the matrix inverse in $K$ is well-defined.
Obviously, $K$ satisfies our definition of reproducing kernel in Section \ref{sec:approximate-pe}.
In fact, the RKHS $\mH$ associated with $K$ is the $N$-dimensional space spanned by $\{\phi_j\}_{j=1}^N$,
and $V_t$ in \eqref{equ:approximated-pe} can be parameterized as $V_t(x)=\phi^T(x)c_t$ for some $c_t\in\mR^N$.
Directly plugging the definitions of $K$ and $V_t$ in \eqref{equ:approximated-pe}, we have the following ODE in $c_t$:
\begin{align}
    \dot c_t = \left(\mE_\mu\left[\phi\phi^T\right]\right)^{-1}\left(\mE_\mu\left[\phi \mT\phi^T\right]c_t+\mE_\mu\left[\phi r\right]\right),\quad c_0\in\mR^N.\label{equ:kernel-vi-projected}
\end{align}
In the above equation, $\mT$ is applied element-wise to $\phi$.
By Theorem \ref{thm:approximated-pe}, \eqref{equ:kernel-vi-projected} admits an equilibrium point $c^*$ satisfying
$\mE_\mu\left[\phi \mT\phi^T\right]c^*+\mE_\mu\left[\phi r\right]=0$.

In addition, by Dynkin's formula, we can replace $\mA \phi$ in \eqref{equ:kernel-vi-projected}  with
\begin{align}
    d \phi_t = \mA \phi_tdt +dM_t,\label{equ:martingale}
\end{align}
where $M_t$ is a vector-valued martingale, and $d\phi_t\triangleq \phi_{t+dt}-\phi_t$ represents the infinitesimal change in the value of $\phi_t$.
By abuse of notation we denote $\phi_t$ and $r_t$ for $\phi(x_t)$ and $r(x_t)$, respectively, when there is no confusion.

\subsection{Continuous-time least squares policy evaluation}\label{sec:batched-learning}



We first replace the integrations with respect to $\mu$ in \eqref{equ:kernel-vi-projected} by their empirical estimations. 
This, together with \eqref{equ:martingale}, leads to  the CT-LSPE algorithm in Algorithm \ref{alg:temporal-differential}.

\begin{algorithm}[t]
    \caption{Continuous-time least squares policy evaluation (CT-LSPE)}
    \label{alg:temporal-differential}

    Choose $N$ basis functions $\phi_j$. Denote $\phi = [\phi_1\ \phi_2\ \cdots\ \phi_N]^T$. 
    Select $\rho >0$.

            Update $c_t$ through the following ODEs:
            \begin{align*}
                \dot c_t &= R_t\int_0^t\phi_s \left(d\phi_s-\gamma \phi_sds\right)^Tc_t+R_t\int_0^t\phi_s r_sds,\quad c_0\in\mR^N,\\
                \dot R_t &= -R_t\phi_t\phi^T_tR_t,\quad R_0=\rho ^{-1} I.
            \end{align*}            
            
\end{algorithm}

The convergence of Algorithm \ref{alg:temporal-differential} relies on the existence of  three functions $\psi_i$, $i=1,2,3$, defined as
\begin{align*}
\psi_1(x) &= \mE_x\left[\int_0^\infty\left(\mE_\mu\left[\phi\mA\phi^T\right]-\phi_t\mA\phi^T_t\right)dt\right],\quad
\psi_2(x) = \mE_x\left[\int_0^\infty\left(\mE_\mu\left[\phi\phi^T\right]-\phi_t\phi^T_t\right)dt\right],\\
\psi_3(x) &= \mE_x\left[\int_0^\infty\left(\mE_\mu\left[\phi r\right]-\phi_tr_t\right)dt\right].
\end{align*}
These functions quantify the differences between the expectations in \eqref{equ:kernel-vi-projected} and their corresponding empirical estimations.

Now, we impose the following assumptions on $x_t$ and $\phi$.
\begin{assumption}[Ergodicity]\label{assumption:ergodic}
    $\{x_t\}_{t\geq0}$ is irreducible, aperiodic, and positive recurrent.
\end{assumption}
\begin{assumption}[Poisson equation]\label{assumption:poisson}
$\psi_i\in\mDA$,  $i=1,2,3$.
\end{assumption}
\begin{assumption}[Ito isometry]\label{assumption:noise}
    $\sup_{t\geq0}\frac{1}{t+1}\mE_x\left[\left|\int_{0}^t\phi_{j}(x_s)dM_s\right|^2\right]<\infty$ for $j=1, 2,\cdots, N$.
\end{assumption}
Assumption \ref{assumption:ergodic} ensures that $x_t$ is an ergodic process.
Assumption \ref{assumption:poisson}  requires that $\psi_i$, $i=1,2,3$, exist and are solutions to certain Poisson equations; see \citet{Glynn1996} for conditions under which $\psi_i$ exists.
Note that the discrete-time versions of Assumptions \ref{assumption:ergodic} and \ref{assumption:poisson} for MDPs have been widely used in ADP and RL literature \citep{Tsitsiklis1994,Tsitsiklis1997}.
Assumption \ref{assumption:noise} essentially requires the second moment of the stochastic integral $\int_{0}^t\phi_j(x_s)dM_s$, $j=1,2,\cdots,N$, does not grow too fast.
Since we can always choose a sufficiently smooth $\phi$ (such as the softmax function, the Gaussian function, or the wavelets) to saturate the noise in the environment, it is possible to satisfy this assumption even for a very noisy process (such as a fat-tailed process).

The convergence of Algorithm \ref{alg:temporal-differential}  is guaranteed in the following theorem.
\begin{theorem}\label{thm:lspe}
Consider Algorithm \ref{alg:temporal-differential}.
    Under Assumptions \ref{assumption:ergodic}, \ref{assumption:poisson}, and \ref{assumption:noise}, 
$c_t$ converges to $c^*$ with probability one. 
In addition, there exist constants $C_i>0$, $i=0,1,2$, such that for a sufficiently large $\epsilon$,
\begin{align}
   \left|c_t-c^*\right|^2\leq \frac{\lambda_M}{\lambda_m}e^{-2(\gamma/\lambda_M-\varepsilon)t+C_1}\left|c_0-c^*\right|^2
   +\frac{ \lambda_MC_2e^{ C_1}}{\varepsilon\lambda_m}\int_{0}^t\frac{1}{s+1}e^{-2(\gamma/\lambda_M-\varepsilon)(t-s)}ds\label{equ:thm-lspe-equ1}
\end{align}
with probability at least $1-C_0/\epsilon^2$, 
where  $\lambda_M$ and $\lambda_m$ represent the largest and smallest eigenvalues of $\mE_\mu[\phi\phi^T]$, respectively, and
$\varepsilon < \gamma/\lambda_M$ is an arbitrary positive constant. 
In particular, \eqref{equ:thm-lspe-equ1} implies that $\left|c_t-c^*\right|=O\left(t^{-0.5}\right)$ with probability one.
\end{theorem}
See Appendix \ref{appendix:proof-lspe} for the proof of Theorem \ref{thm:lspe} and the detailed formulation of the error bound.

The first term on the right-hand-side of \eqref{equ:thm-lspe-equ1} represents the error induced from the prior guess $ c_0$.
The second term is driven by the estimation error induced from the empirical estimation of the expectations in \eqref{equ:kernel-vi-projected}.
By the law of large numbers, this estimation error  converges to $0$ with rate $O\left(t^{-0.5}\right)$.
Since $t^{-0.5}$ decreases slower than the exponential rate, it dominates the error bound of $\left|c_t-c^*\right|$.
Note that the obtained convergence rate is consistent with the  result  in \citet{Yu2009} for discrete-time LSPE.


\subsection{Continuous-time temporal-differential learning}\label{sec:td-learning}
An alternative way to approximate \eqref{equ:kernel-vi-projected} via on-line data is to use stochastic approximation \citep{Benveniste1990}.
In this case, $\mE_\mu\left[\phi \mT\phi^T\right]c_t+\mE_\mu\left[\phi r\right]$
is directly replaced by the on-line data $\phi_tc_t^Td\phi_t+\phi_tr_t$ at time $t$.
This, together with \eqref{equ:martingale}, leads to the CT-TD learning algorithm in Algorithm \ref{alg:temporal-differential-infty}.
To accommodate the noise induced from stochastic approximation, the matrix inverse in \eqref{equ:kernel-vi-projected} is replaced by a slowly decreasing step size $\alpha_t$.
In particular, we name 
\begin{align*}
c_t^T\frac{d}{dt}\phi_t-\gamma c_t ^T\phi_t+ r_t
\end{align*}
as the {\em temporal-differential} error.
\begin{algorithm}[t]
    \caption{Continuous-time temporal-differential (CT-TD) learning}
    \label{alg:temporal-differential-infty}

    Choose $N$ basis functions $\phi_j$. Denote $\phi = [\phi_1\ \phi_2\ \cdots\ \phi_N]^T$. 
    $\alpha_t>0$  satisfies $\dot\alpha_t\leq0$, $\int_0^\infty \alpha_t^2dt<\infty$, and $\int_0^\infty \alpha_tdt=\infty$.

            Update $c_t$ through the following SDE:
        \begin{align*}
            dc_t =\alpha_t\phi_t  \left(c_t^Td\phi_t-\gamma c_t ^T\phi_tdt + r_tdt\right),\quad c_0\in\mR^N.
        \end{align*}            

\end{algorithm}

\begin{assumption}[Finite moments]\label{assumption:poisson-norm}
There exist $C>0$ and real-valued function $\varphi$, such that 
$\mE_\mu\left(|\varphi|^4\right)<C$,  
$\mE_\mu\left(|\psi_i|^4\right)<C$,  
$\mE_\mu\left(|\mA\psi_i|^4\right)<C$, 
$\left(d\phi_j(x_t)\right)^2\leq \varphi(x_t)dt$, and
 $\left(d\psi_{i,j}(x_t)\right)^2\leq \varphi(x_t)dt$, for the $j$-th element in $\phi$ and $\psi_i$,  and $i=1,2,3$.
\end{assumption}
A discrete-time version of the boundedness condition on the high-order moments of $\psi_i$ and $\mA\psi_i$ is also required in \citet{Tsitsiklis1997} and \citet{Benveniste1990}.
The conditions on $\left(d\phi_j(x_t)\right)^2$ and $\left(d\psi_{i,j}(x_t)\right)^2$ are new in this paper and essentially require $\psi_i(x_t)$ and $\phi_t$ do not vary too fast.
\begin{theorem}\label{thm:td}
Consider Algorithm \ref{alg:temporal-differential-infty}.
    Under Assumptions \ref{assumption:ergodic}, \ref{assumption:poisson}, and \ref{assumption:poisson-norm}, 
$c_t$ converges to $c^*$ with probability one.
In addition, there exist constants $C_0$ and $C_1$, such that for sufficiently large $\epsilon$ and  integer $k$,
\begin{align}
\left|c_t-c^*\right|^{2}<\epsilon\left(\left| c_0-c^*\right|^2+\sqrt{A_kC_0}\left( 2^{\epsilon}-1\right)e^{2C_1 k}\right)e^{-2 \gamma \int_{0}^{t} \alpha_{s}ds} \label{equ:thm-td-equ1}
\end{align}
with probability at least $1-A_kC_0-\frac{2}{\epsilon}$, 
where    
\begin{align*}
    A_k = \alpha_k^2 + \int_k^{\infty}\alpha_t^2dt + \left(\int_k^{\infty}\alpha_t^2dt\right)^2.
    \end{align*}
In particular, \eqref{equ:thm-td-equ1} implies that $\left|c_t-c^*\right|=O\left(e^{-\gamma \int_{0}^{t} \alpha_{s}ds}\right)$ with probability one.
\end{theorem}
See Appendix \ref{appendix:proof-td} for the proof of Theorem \ref{thm:td} and the detailed formulation of the error bound.

Note that as $k$ goes to the infinity, $A_k$ converges to $0$ by the conditions on $\alpha_t$ in Algorithm \ref{alg:temporal-differential-infty}.
In addition, the convergence rate of $A_k$ is slower than the exponential rate, due to the slow decreasing rate of $\alpha_t$.
As a result, $\lim_{k\rightarrow \infty}\sqrt{A_k}e^{2C_1 k}=\infty$, and the error bound in \eqref{equ:thm-td-equ1} grows with $k$.

A key difference between Algorithms  \ref{alg:temporal-differential} and \ref{alg:temporal-differential-infty} is that the time integration of the whole time series is incorporated in Algorithm \ref{alg:temporal-differential}, while only the on-line data at current time is used in Algorithm \ref{alg:temporal-differential-infty}.
As a result, Algorithm \ref{alg:temporal-differential-infty} may produce a much noisier estimation on the value function, especially at the beginning of the learning process. 
Indeed, the experimental results in Section \ref{sec:ou-process} have shown that the point estimations  produced by Algorithm \ref{alg:temporal-differential} have smaller standard errors.
In fact, to ensure the convergence of Algorithm \ref{alg:temporal-differential-infty}, a more restrictive condition (Assumption \ref{assumption:poisson-norm}) is required.
In contrast, the presence of decreasing step size $\alpha_t$ in Algorithm \ref{alg:temporal-differential-infty} provides more freedom to adjust the convergence rate of the CT-TD learning.
For instance, if $\alpha_t=(a+bt^\beta)^{-1}$ with $a>0$,  $b>0$ and  $0.5<\beta\leq1$,
then we have by Lemma \ref{lem:step} in Appendix \ref{appendix:lemma-proof} that $\left|c_t-c^*\right|=O(t^{- \gamma \beta})$  for sufficiently large $t$.

We can see from the proofs of Algorithms  \ref{alg:temporal-differential} and \ref{alg:temporal-differential-infty} that the constants in the error bounds in Theorem \ref{thm:lspe} and Theorem \ref{thm:td} are polynomials of $N$, that is, the number of bases. 
As a result, if $N$ increases, the error bounds in Theorems  \ref{thm:lspe} and \ref{thm:td}, which represent the estimation errors, increase as well.
On the other hand, a small $N$ also leads to a large divergence from $\mH$ to $L^2_\mu$, which is quantified by $\left\|(\Pi-\mID)(\mT V_s+r)\right\|_{\mu}$ in Theorem \ref{thm:approximated-pe}.
As a result, a small $N$ causes a large approximation error.
This trade-off between approximation error and estimation error, or better known as the bias-variance trade-off, is a well-known phenomenon in various machine learning methods. 

Finally, note from Algorithms  \ref{alg:temporal-differential} and \ref{alg:temporal-differential-infty} that the state $x_t$ of the environment does not appear explicitly in the updating equations. 
Instead, the updating equations only depend on $r(x_t)$ and $\phi(x_t)$. 
From the perspective of control theory,  the pair $(\phi(x), r(x))$ can be viewed as the  output of the environment at state $x$.
Therefore, Algorithms  \ref{alg:temporal-differential} and \ref{alg:temporal-differential-infty} can also be applied to output-feedback control problems and partially observable Markov processes.


\section{Computer-based Experiments}
In this section, we present two computer-based experiments to illustrate the effectiveness of the two learning algorithms presented in the previous section.
\subsection{Continuous-time ARMA(2,1) process}\label{sec:ou-process}
Consider $\{y_t\}_{t\geq0}$ as a real-valued continuous-time ARMA(2,1) process governed by
\begin{align*}
\ddot y_t  + \theta_2 \dot y_t + \theta_1  y_t =  \dot \xi_{1,t} + \sigma_1\xi_{1,t}  + \sigma_2\xi_{2,t} +  \sigma_3\xi_{3,t},
\end{align*}
where  $\theta_1$, $\theta_2$, $\sigma_1$, $\sigma_2$, and $\sigma_3$ are positive reals, and $\xi_{i,t}$, $i=1,2,3$, are three independent white noises \citep[Section 3.2]{Arnold1974}.
Denote $x_t=(y_t,\dot y_t-\xi_{1,t})$. 
Then, we can derive the following two-dimensional Ornstein–Uhlenbeck (OU) process:
\begin{align*}
    dx_t  = Ax_tdt + \Sigma dw_t,\quad
    A  = 
    \begin{bmatrix}
        0&1\\
        -\theta_1&-\theta_2
    \end{bmatrix},\quad 
    \Sigma = 
    \begin{bmatrix}
        1&0&0\\
        \sigma_1-\theta_2&\sigma_2&\sigma_3\\
    \end{bmatrix},
\end{align*}
where $w_t = [w_{1,t}\ w_{2,t}\ w_{3,t}]^T$ is a vector of three independent Brownian motions defined with $\dot w_i=\xi_i$, $i=1,2,3$ \citep[Eq. (3.2.3)]{Arnold1974}. 

We know from \citet{Kushner1967a} that the discounted value function $V^*$ defined by \eqref{equ:cost-discount} solves the following partial differential equation:
\begin{align}
   0 = -\gamma V^*(x) + \partial_x V^*(x)Ax+ \frac{1}{2}\sum_{i,j}\partial_{x_ix_j}^2 V^*(x)(\Sigma\Sigma^T)_{i,j} + r(x),\label{equ:discount-Lyapunov}
\end{align}
where $(\Sigma\Sigma^T)_{i,j}$ denotes the $(i,j)$-th element in $\Sigma\Sigma^T$.


%


\begin{wrapfigure}{r}{0.29\textwidth}
\centering
\vspace{-10pt}
           \includegraphics[width=0.28\textwidth]{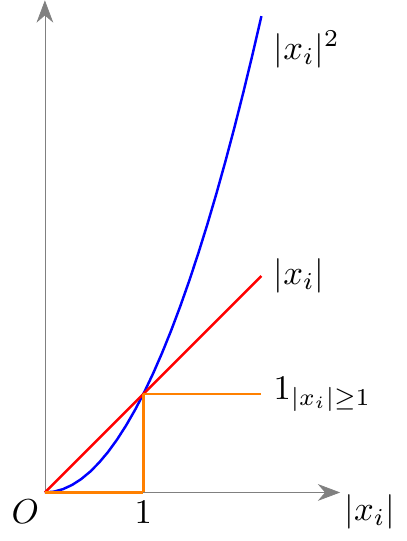}
    \caption{Comparison of different running-costs.}\label{fig:cost}
    \vspace{-10pt}
\end{wrapfigure}
In this example, we aim at using
Algorithms \ref{alg:temporal-differential} and  \ref{alg:temporal-differential-infty} to estimate  
 the discounted value functions with respect to three types of running-costs. 
First, a quadratic function $r(x)=10\sum_{i=1}^2|x_i|^2$ 
is used, under which \eqref{equ:discount-Lyapunov} reduces to a Lyapunov equation.
Then, we select the $L^1$ vector norm $r(x)=10\sum_{i=1}^2|x_i|$.
Finally, we choose $r(x)=10\sum_{i=1}^21_{|x_i|\geq1}$, which is essentially the 0-1 loss.
See Figure \ref{fig:cost} for an illustration of different choices of $r$.
In all three cases, the discounting factor is fixed at $\gamma=1$, and the step size $\alpha_t$ in Algorithm \ref{alg:temporal-differential-infty} is chosen as $\alpha_t = 0.7/(1+t)$.
Model parameters of the OU process are chosen as $\theta_1=\theta_2=\sigma_1=\sigma_2=\sigma_3=1$, and the initial state  is chosen as $x_0=(1,0)$.
Six basis functions are selected: $\phi(x) = [1\ x_1\ x_2\ x_1^2\ x_1x_2\ x_2^2]$, where $x=(x_1,x_2)$.

\begin{figure}
    \centering
    \begin{subfigure}[t]{\textwidth}
        \includegraphics[width=\textwidth]{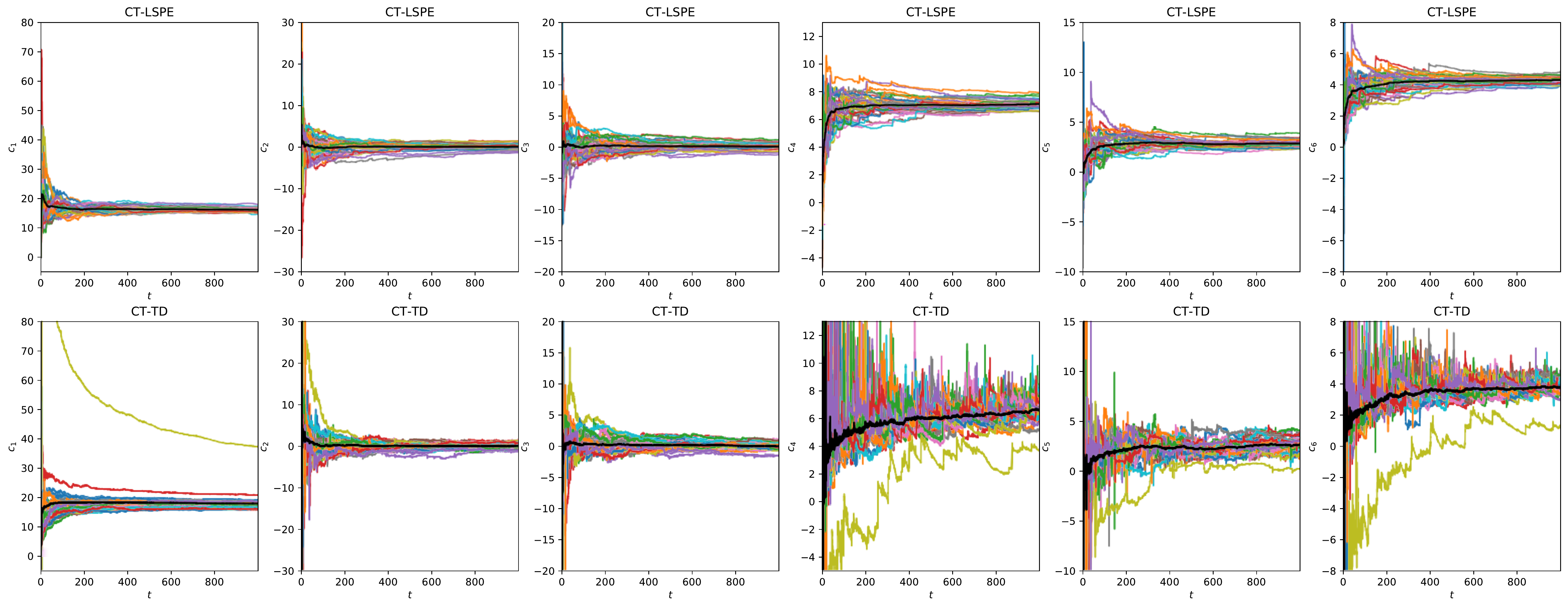}
        \caption{$r(x)=10\sum_{i=1}^2|x_i|^2$.}
        \label{fig:weights-paths-l2}
    \end{subfigure}

        \begin{subfigure}[t]{\textwidth}
        \includegraphics[width=\textwidth]{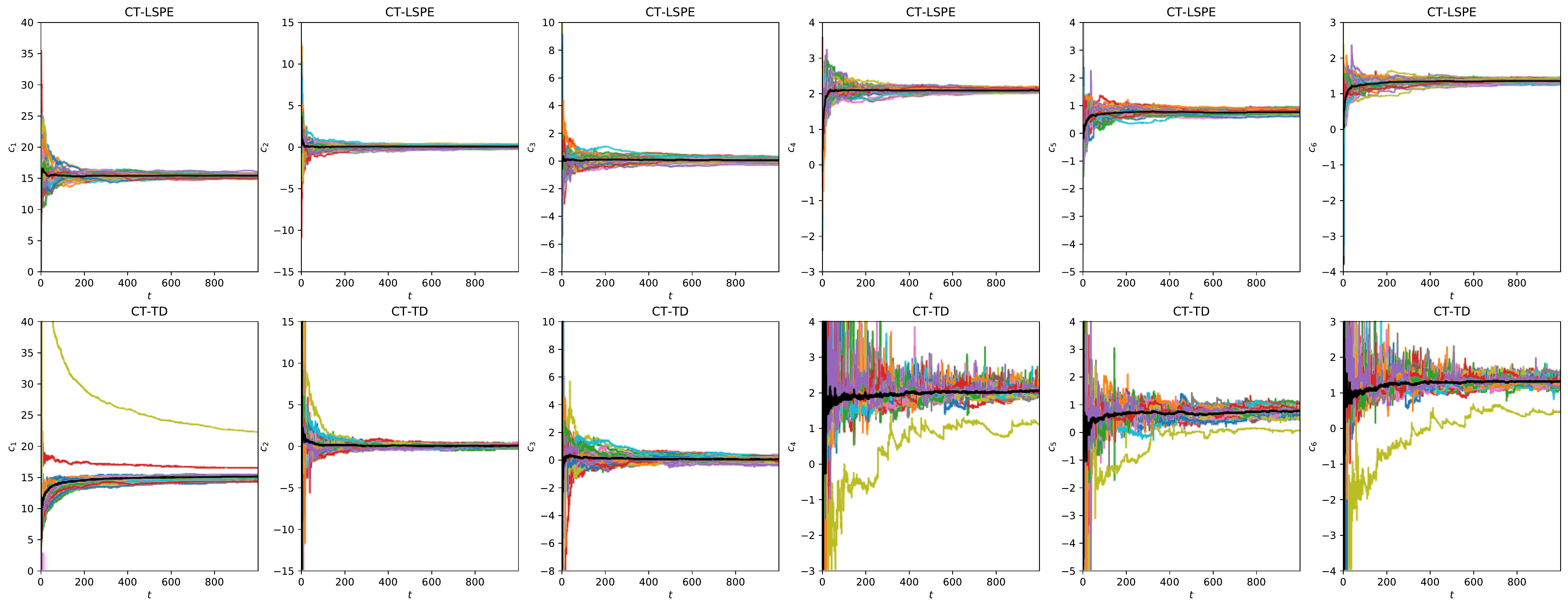}
        \caption{$r(x)=10\sum_{i=1}^2|x_i|$.}
        \label{fig:weights-paths-l1}
    \end{subfigure}


     \begin{subfigure}[t]{\textwidth}
        \includegraphics[width=\textwidth]{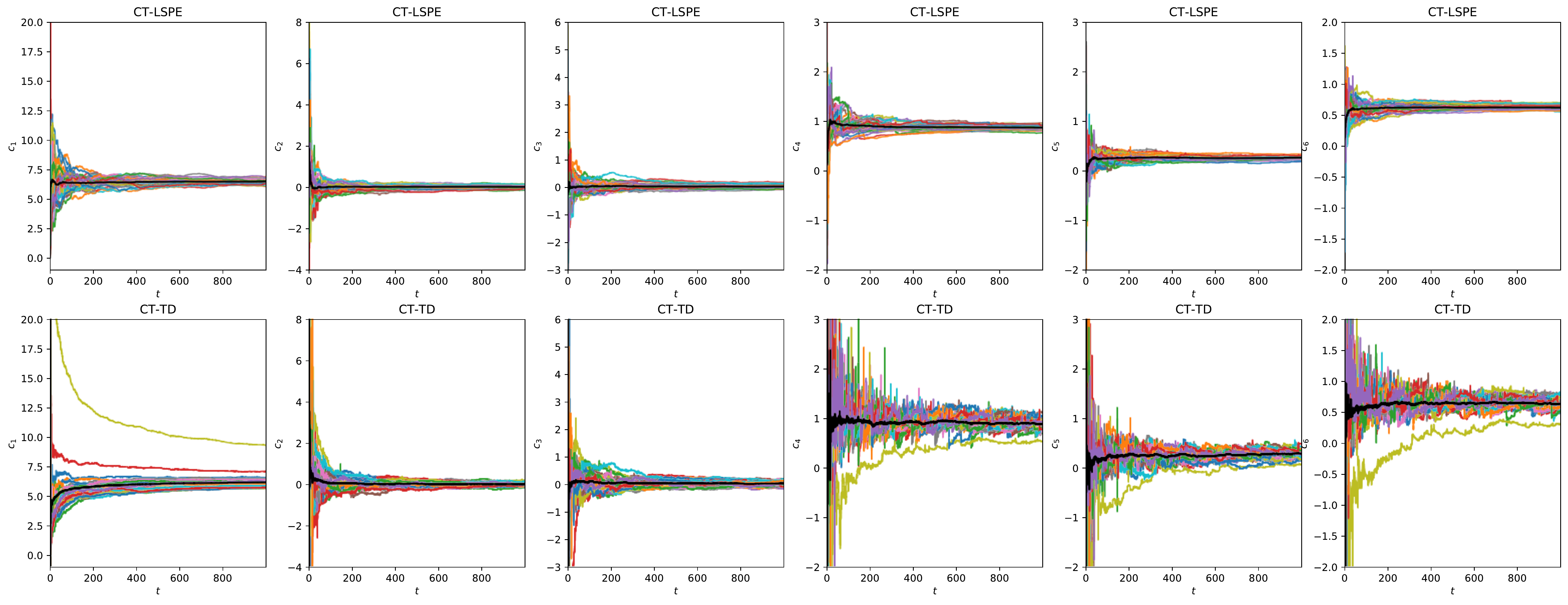}
        \caption{$r(x)=10\sum_{i=1}^21_{|x_i|\geq1}$.}
        \label{fig:weights-paths-l1-delta}
    \end{subfigure}

    \caption{Value estimation for a continuous-time ARMA(2,1) process. 
    The weights of basis functions ($\{c_{i,t}\}_{t\geq0}$ for $i=1,2,\cdots,6$) along the 35 sample paths are plotted for the three choices of $r$. 
    Each color, except black, is corresponding to one sample path.
    The solid black lines represent the point estimations of $c_{i,t}$.}
    \label{fig:experiment1-weights-paths}
\end{figure}

Figures \ref{fig:experiment1-weights-paths}  and \ref{fig:experiment1} show the experimental results of our algorithms.
To evaluate the performance of the algorithms, we conduct learning processes over $35$ sample paths. 
Figure \ref{fig:experiment1-weights-paths} illustrates the $35$ paths of the weights $\{c_i\}_{i=1}^6$ corresponding to the six basis functions.
Figures  \ref{fig:weights-series-box-l2},  \ref{fig:weights-series-box-l1}, and  \ref{fig:weights-series-box-l1-delta}
show the box plots of these weights at $t=1000$.
Overall, the weights obtained from both the CT-LSPE algorithm  (Algorithm  \ref{alg:temporal-differential}) and the CT-TD learning algorithm (Algorithm  \ref{alg:temporal-differential-infty}) converge to same values in all three cases.
However, compared with the CT-LSPE, the CT-TD learning generates more dispersed samples, and the outliers are scattered further from the center of the clusters of samples.
This is not surprising, as the time integration in the CT-LSPE tends to smooth out the noise in the learning process, especially at the beginning of the learning phase.
Indeed, we can see from Figure \ref{fig:experiment1-weights-paths} that CT-TD learning generates noisier paths of basis-function weights.
To further illustrate the effectiveness of the proposed algorithms, we compare the value functions obtained from the two learning algorithms with that from Monte-Carlo (MC) simulations in Figures \ref{fig:mc-series-l2}, \ref{fig:mc-series-l1}, and \ref{fig:mc-series-l1-delta}.
The point estimations of value functions at $t=1000$ with $95\%$ confidence intervals are given in Figures \ref{fig:mc-l2},  \ref{fig:mc-l1}, and \ref{fig:mc-l1-delta}. 
Note that in all three cases, these point estimations  are approximately at the same level.
In addition, compared with  the CT-TD learning, the CT-LSPE produces a smaller standard error.
This result is also consistent with the observations in Figures  \ref{fig:weights-series-box-l2},  \ref{fig:weights-series-box-l1}, and  \ref{fig:weights-series-box-l1-delta}.
Finally, the MC simulation produces the largest standard error.

\begin{figure}[t]
    \centering
      \begin{subfigure}[t]{0.31\textwidth}
      \centering
        $r(x)=10\sum_{i=1}^2|x_i|^2$
    \end{subfigure}
    ~
    \begin{subfigure}[t]{0.31\textwidth}
    \centering
        $r(x)=10\sum_{i=1}^2|x_i|$
    \end{subfigure}
    ~
    \begin{subfigure}[t]{0.31\textwidth}
    \centering
        $r(x)=10\sum_{i=1}^21_{|x_i|\geq1}$
    \end{subfigure}
    

       \begin{subfigure}[t]{0.31\textwidth}
        \includegraphics[width=\textwidth]{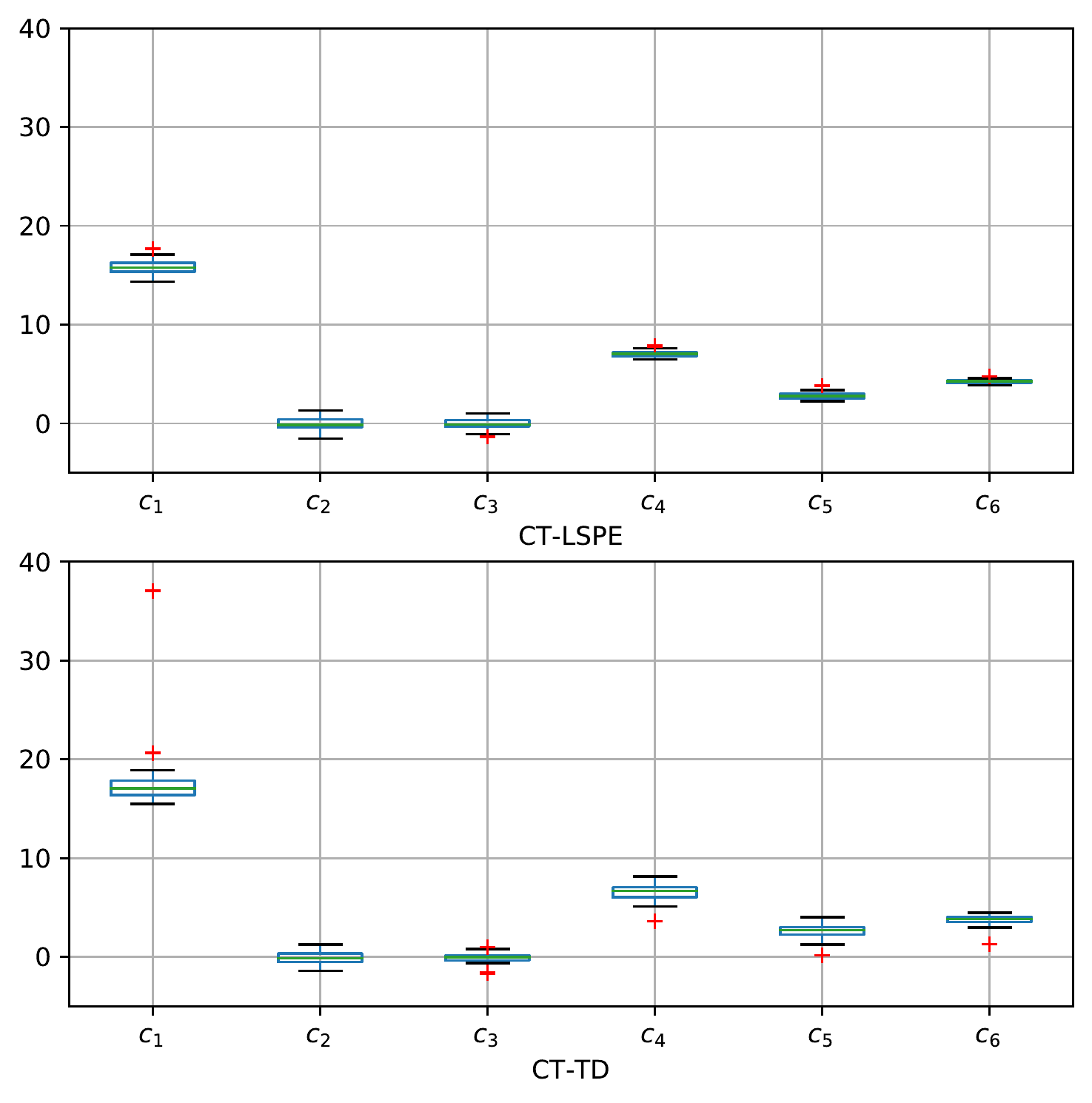}
        \caption{}
        \label{fig:weights-series-box-l2}
         \end{subfigure}
         ~
         \begin{subfigure}[t]{0.31\textwidth}
        \includegraphics[width=\textwidth]{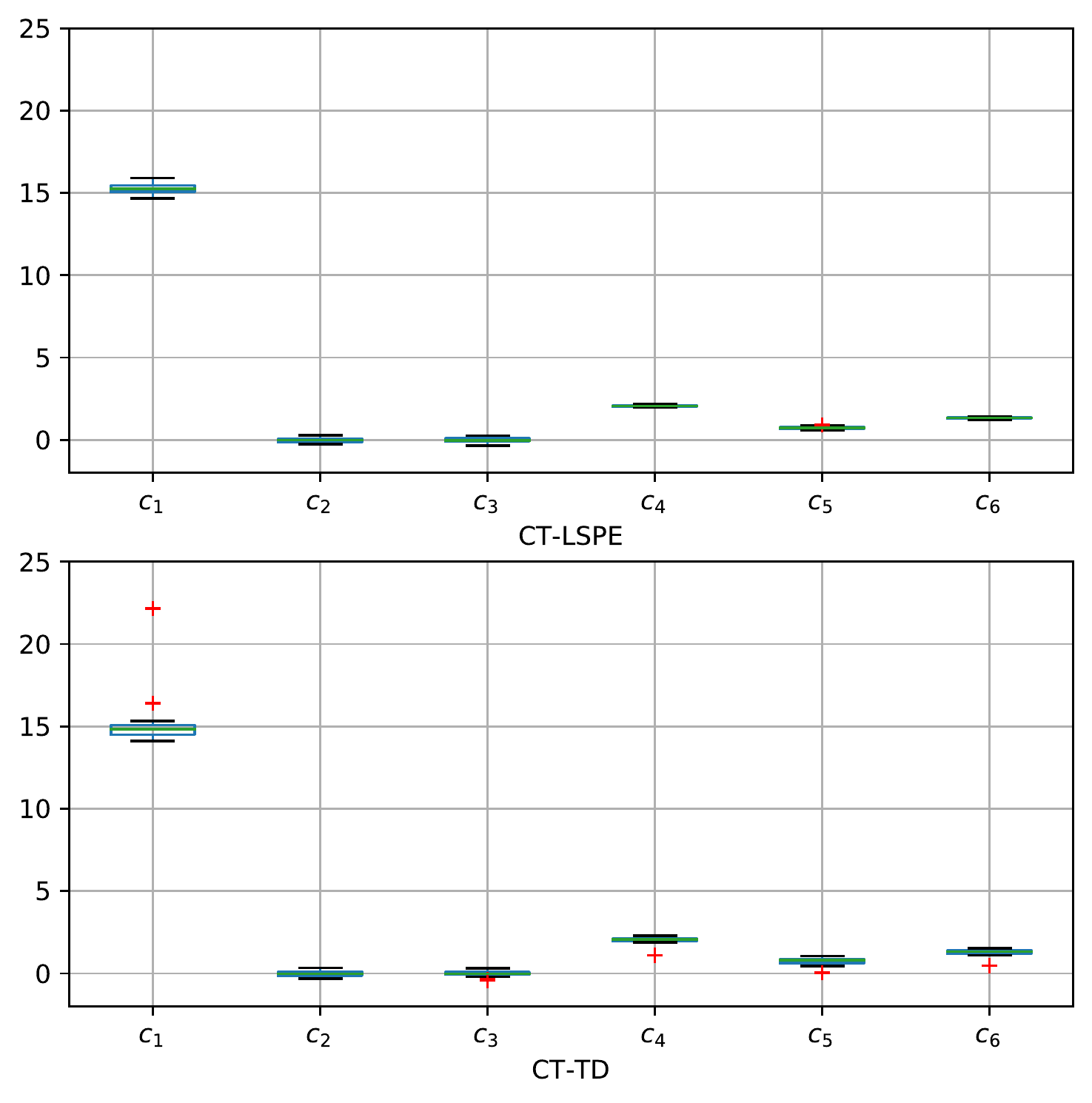}
        \caption{}
        \label{fig:weights-series-box-l1}
    \end{subfigure}
        ~
       \begin{subfigure}[t]{0.305\textwidth}
        \includegraphics[width=\textwidth]{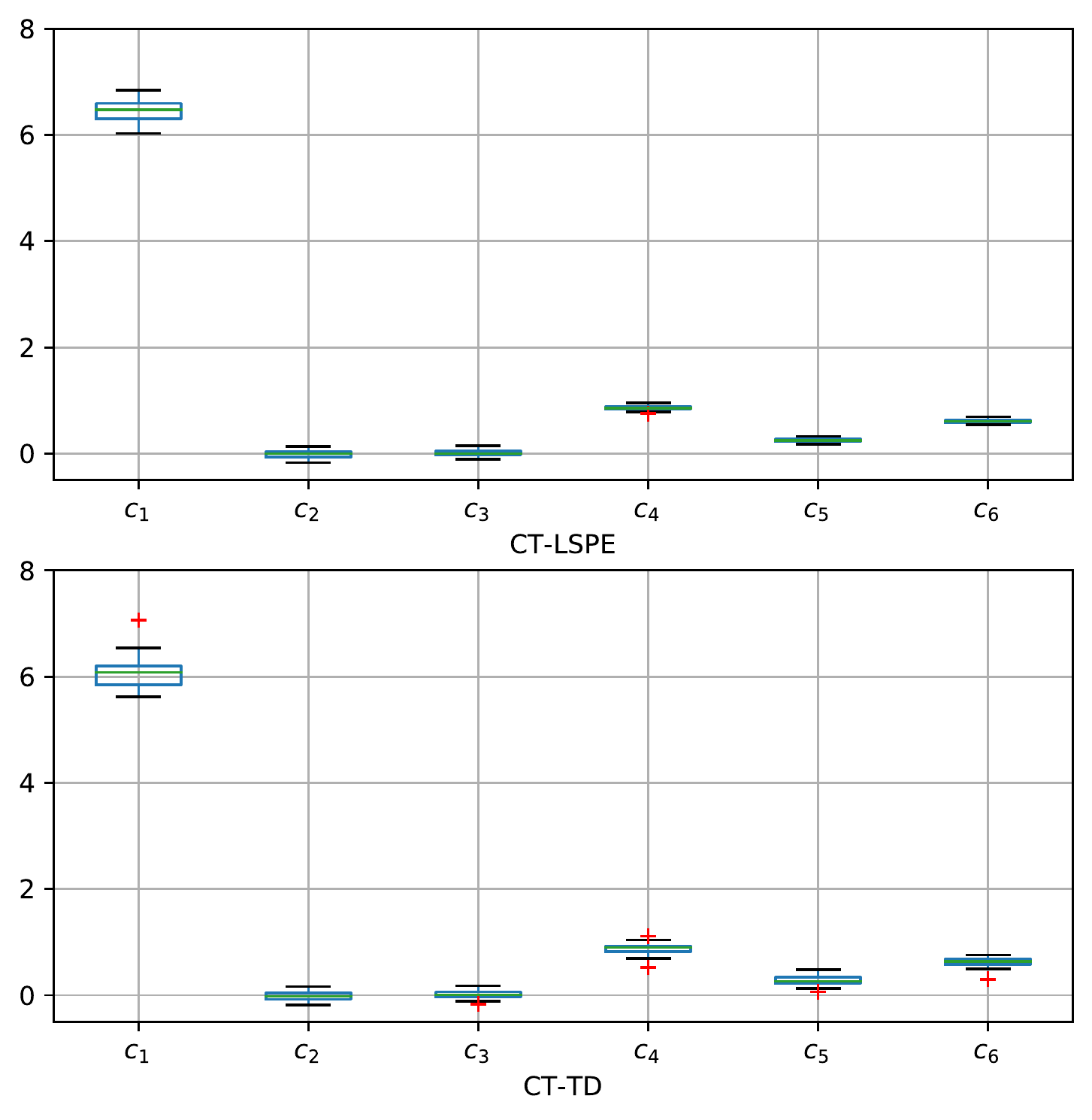}
        \caption{}
        \label{fig:weights-series-box-l1-delta}
            \end{subfigure}

        \begin{subfigure}[t]{0.31\textwidth}
        \includegraphics[width=\textwidth]{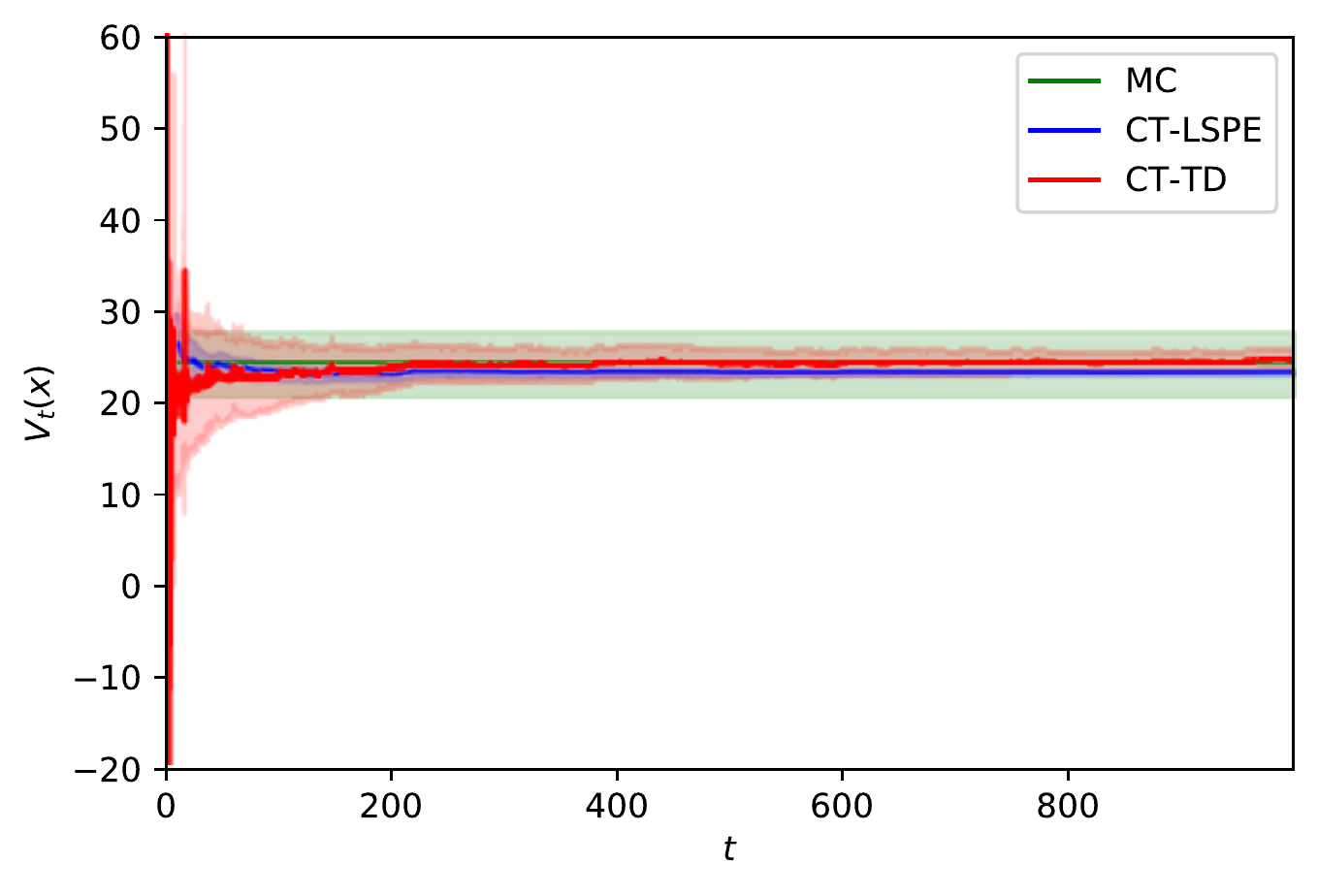}
        \caption{}
        \label{fig:mc-series-l2}
         \end{subfigure}
    ~
        \begin{subfigure}[t]{0.31\textwidth}
        \includegraphics[width=\textwidth]{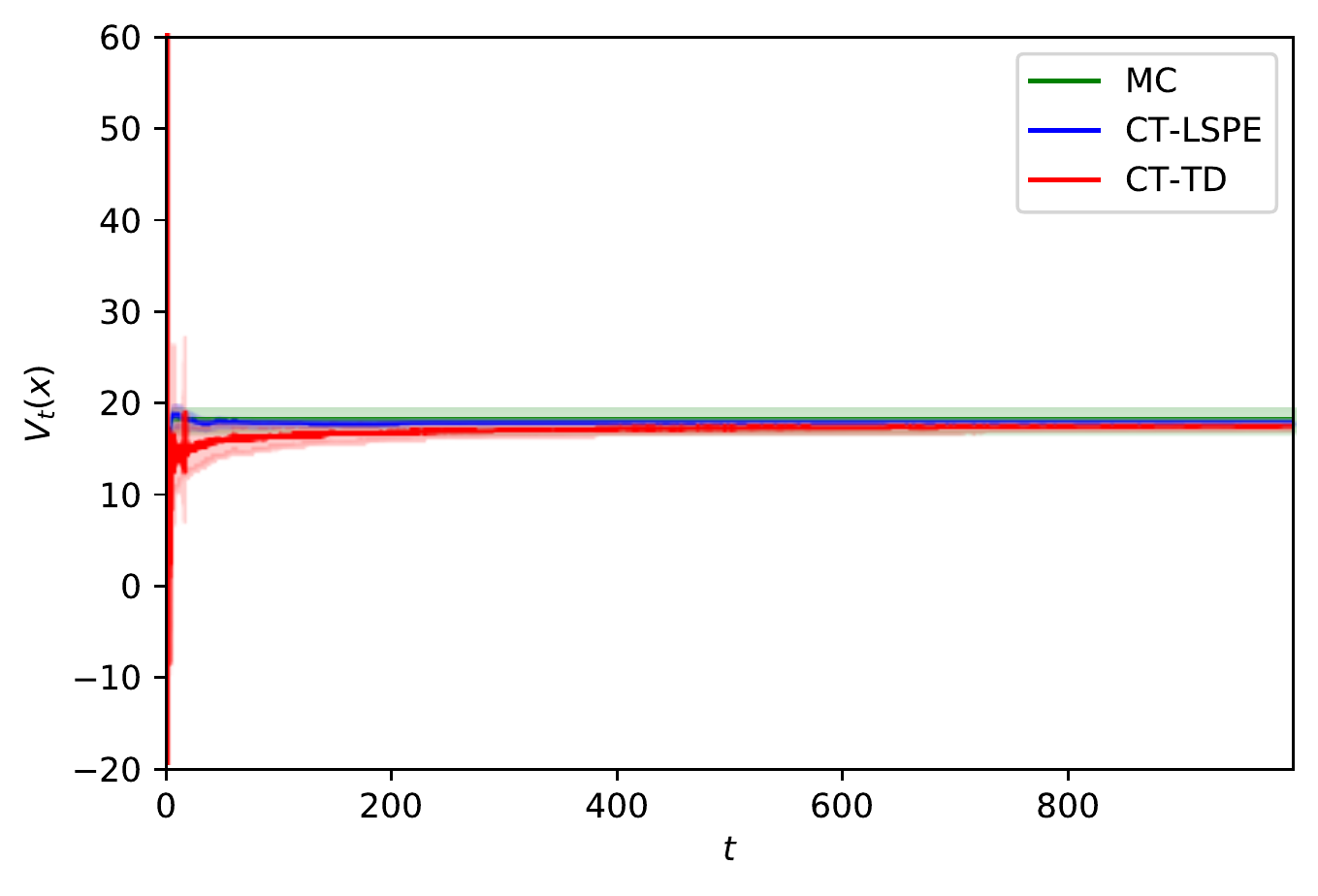}
        \caption{}
        \label{fig:mc-series-l1}
    \end{subfigure}
    ~
       \begin{subfigure}[t]{0.31\textwidth}
        \includegraphics[width=\textwidth]{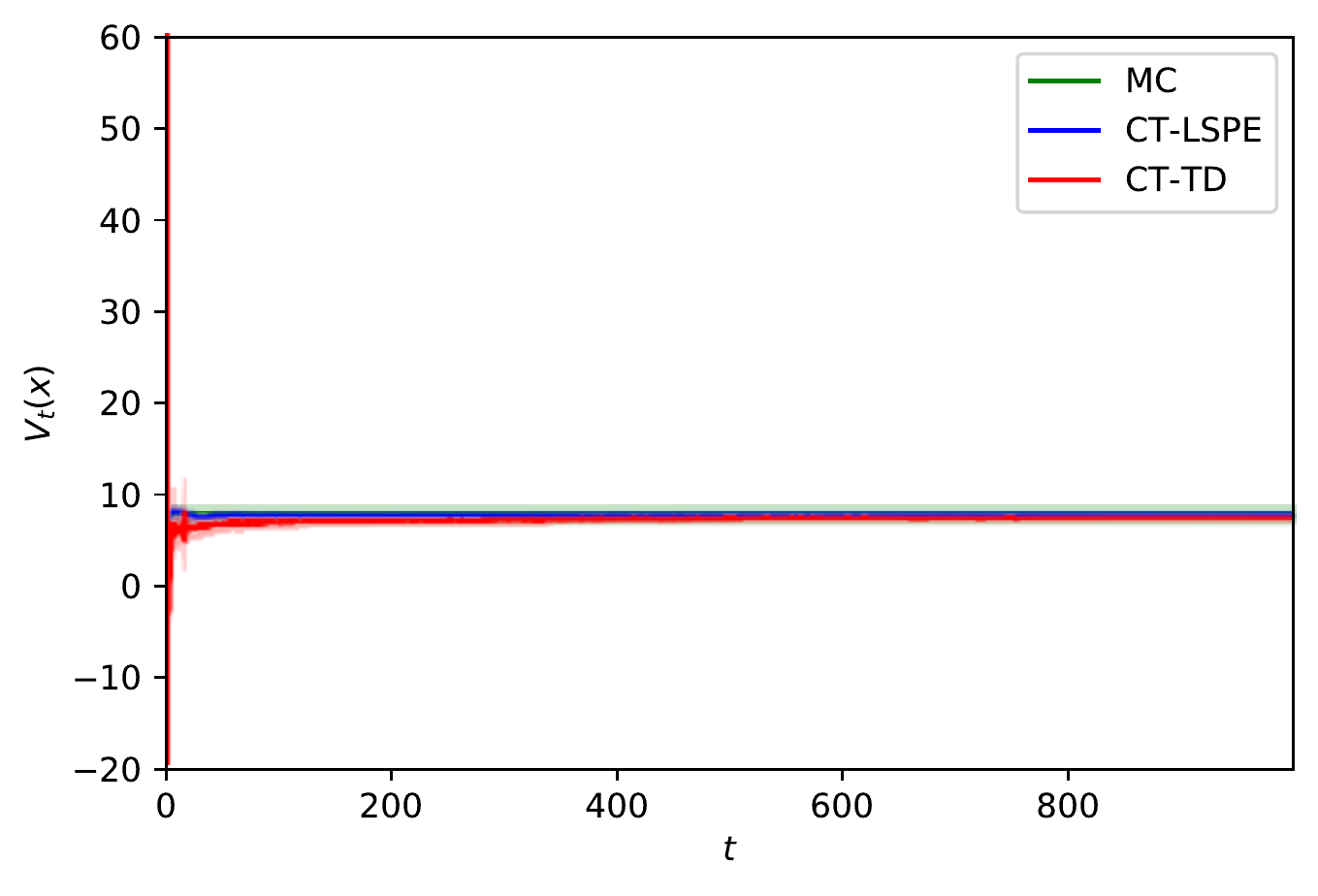}
        \caption{}
        \label{fig:mc-series-l1-delta}
            \end{subfigure}

         \begin{subfigure}[t]{0.303\textwidth}
        \includegraphics[width=\textwidth]{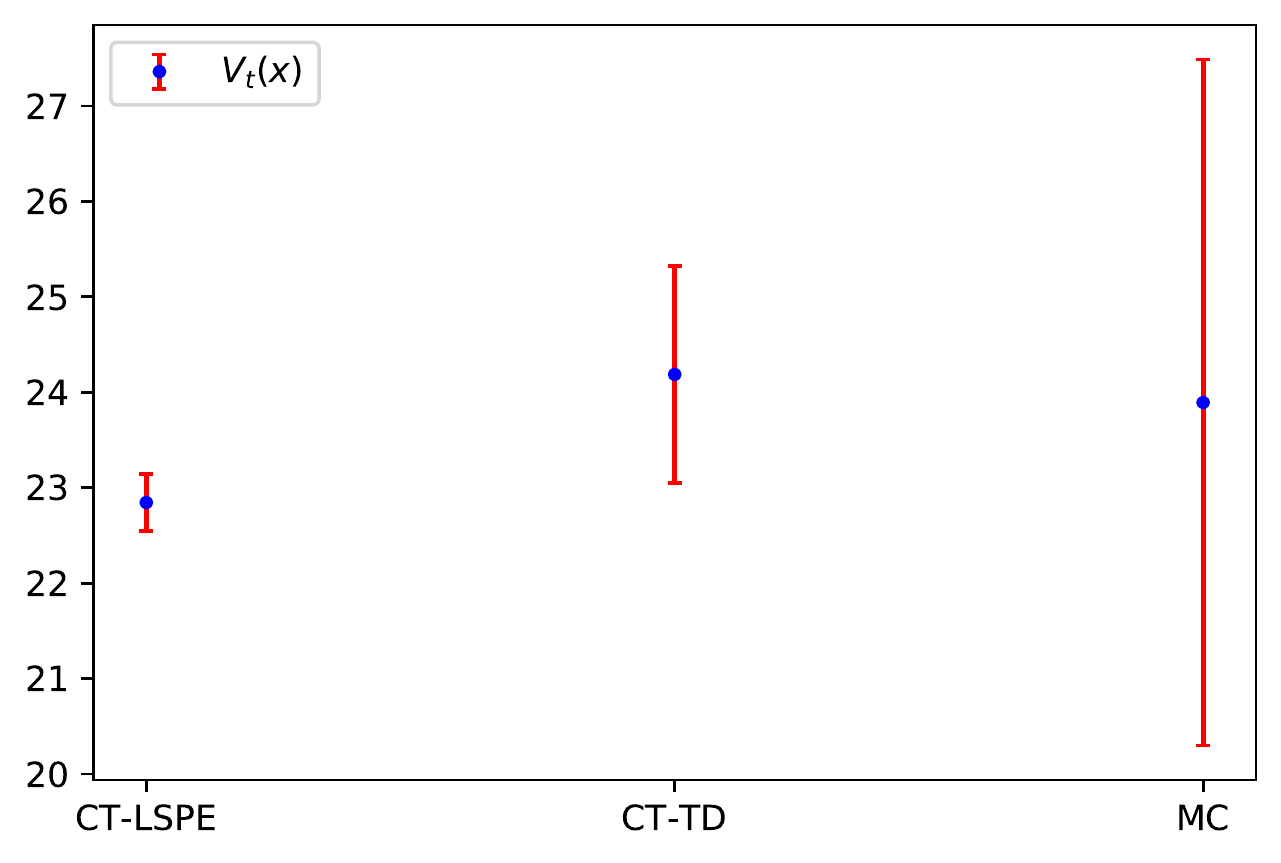}
        \caption{}
        \label{fig:mc-l2}
         \end{subfigure}
    ~
        \begin{subfigure}[t]{0.31\textwidth}
        \includegraphics[width=\textwidth]{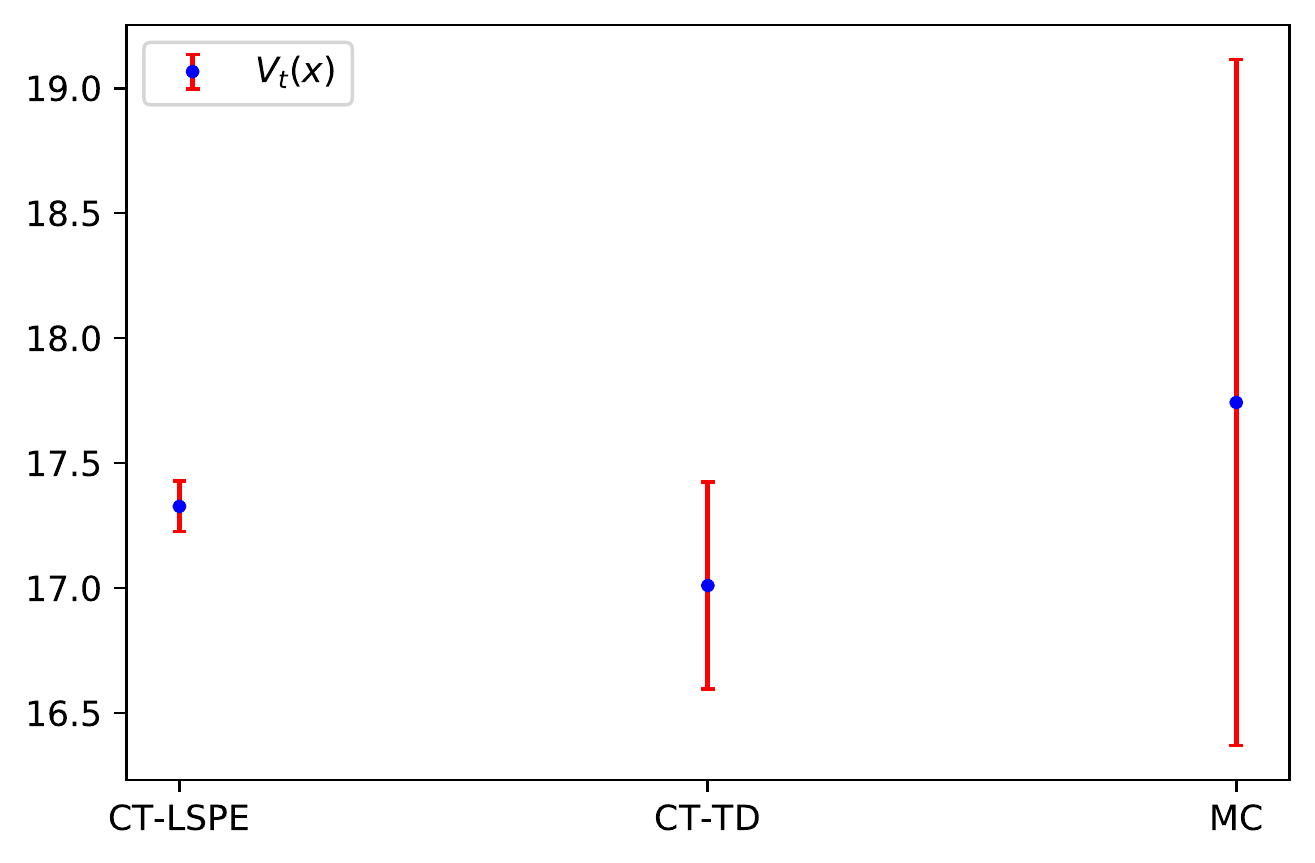}
        \caption{}
        \label{fig:mc-l1}
    \end{subfigure}
~
        \begin{subfigure}[t]{0.31\textwidth}
        \includegraphics[width=\textwidth]{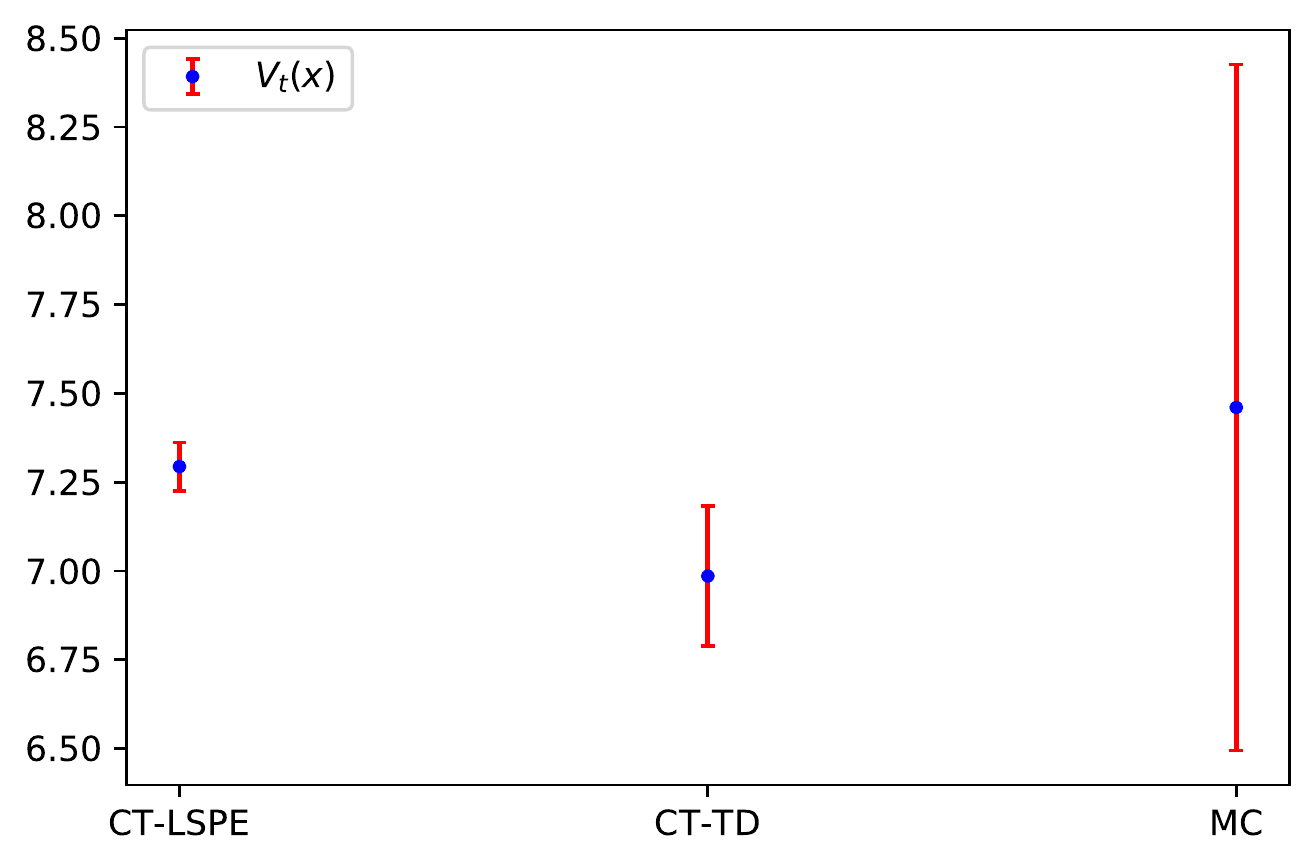}
        \caption{}
        \label{fig:mc-l1-delta}
    \end{subfigure}
    \caption{Value estimation for a continuous-time ARMA(2,1) process. 
    Each one of the three columns in the figure is corresponding to one choice of $r$.
       Figures (a), (b), and (c)  show the box plots  of the  basis-function weights ($\{c_{i,t}\}_{t\geq0}$ for $i=1,2,\cdots,6$) at $t=1000$ along the $35$ sample paths.
       Figures (d), (e), and (f)  show the estimations of the value functions ($V_t(x)$) at $x=(1,0)$.
The solid lines represent point estimations of the true values using different algorithms.
    The shaded areas represent $95\%$ confidence intervals.
Figures (g), (h), and (i)  show the point estimations (blue dots) of $V_{1000}(x)$ at $x=(1,0)$ with $95\%$ confidence intervals (red bars) under different algorithms.
    }
    \label{fig:experiment1}
\end{figure}

\subsection{Control of the benchmark double inverted pendulum on a cart}
In this example, we apply the CT-LSPE together with the actor-critic algorithm \citep{Konda2000} to improve the 
control performance of a double inverted pendulum on a cart (DIPC) (Figure \ref{fig:DIPC}).
It is well-known that the DIPC is a nonlinear under-actuated plant,  and designing a controller for DIPC is a challenging task in 
the field of nonlinear control \citep{Khalil2015}.


\begin{figure}[h]
\centering
           \includegraphics[width=0.4\textwidth]{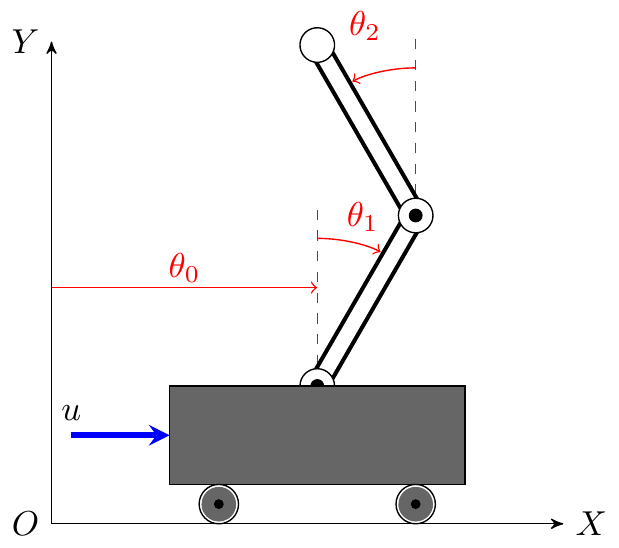}
    \caption{Double inverted-pendulum on a cart.}\label{fig:DIPC}
\end{figure}

The mathematical model of the DIPC is given below \citep{Moysis2016}, which is derived
 from the Lagrange equations of the DIPC:
\begin{align*}
D(\theta) \ddot{\theta}+{C}(\theta, \dot{\theta}) \dot{\theta}+{G}(\theta)=H u,
\end{align*}
where $\theta = [\theta_0\ \theta_1\ \theta_2]^T$, $\theta_0$ denotes the horizontal displacement of the cart, $\theta_1$ and $\theta_2$ denote the angles of lower and upper pendulum links with respect to the vertical position, respectively, $u$ is an input force applied to the cart, and
\begin{align*}
D(\theta) & =
\begin{bmatrix}
C_1 & C_2 \cos \theta_{1} & {C_{3} \cos \theta_{2}} \\ 
{C_{2} \cos \theta_{1}} & {C_{4}} & {C_{5} \cos \left(\theta_{1}-\theta_{2}\right)} \\ 
C_{3} \cos \theta_{2} & C_{5} \cos \left(\theta_{1}-\theta_{2}\right) & {C_{6}}
\end{bmatrix}, \quad
{G}(\theta)  =
\begin{bmatrix}
{0} \\ {-C_7 \sin \theta_{1}} \\ {-C_8 \sin \theta_{2}}
\end{bmatrix},\\
{C}(\theta, \dot{\theta}) & =
\begin{bmatrix}
{0} & {-C_{2} \sin \left(\theta_{1}\right) \dot{\theta}_{1}} & {-C_{3} \sin \left(\theta_{2}\right) \dot{\theta}_{2}} \\ 
{0} & {0} & {C_{5} \sin \left(\theta_{1}-\theta_{2}\right) \dot{\theta}_{2}} \\ 
0 & -C_{5} \sin \left(\theta_{1}-\theta_{2}\right) \dot{\theta}_{1} & {0}
\end{bmatrix}, \quad  H  = 
\begin{bmatrix}
1\\ 0 \\ 0
\end{bmatrix},\\
C_{1} & =m_{0}+m_{1}+m_{2}, \quad
C_{2} =\left(\frac{1}{2} m_{1}+m_{2}\right) l_1, \quad
C_{3} =\frac{1}{2} m_{2} l_2,  \quad
C_{4} =\left(\frac{1}{3} m_{1}+m_{2}\right) l_1^{2},\\  
C_{5} &=\frac{1}{2} m_{2} l_1 l_2, \quad
C_{6} =\frac{1}{3} m_{2} l_2^{2}, \quad
C_7=\left(\frac{1}{2} m_{1}+m_{2}\right) l_1 g, \quad
C_8 =\frac{1}{2} m_{2} l_2 g,
\end{align*}
The model parameters are listed in Table \ref{tab:dipc}.
\begin{table}[htp]
\caption{Parameters of the DIPC.}
\begin{center}
\begin{tabular}{l|l|l}
\hline
Parameters & Value & Definitions\\\hline
$m_0$ & 1.5kg &  Weight of the cart\\
$m_1$ & 0.5kg & Weight of the lower pendulum link\\
$m_2$ & 0.75kg & Weight of the upper pendulum link\\
$l_1$ & 0.5kg & Length of the lower pendulum link\\
$l_2$ & 0.75kg & Length of the upper pendulum link\\
$g$ & 9.81m/s$^2$ & Gravity constant\\
\hline
\end{tabular}
\end{center}
\label{tab:dipc}
\end{table}%

Denote the state of the DIPC as $x=[\theta_0\ \theta_1\ \theta_2\ \dot\theta_0\ \dot\theta_1\ \dot\theta_2]$.
Now, we can rewrite the above model in the standard state-space form:
\begin{align*}
\dot x = 
\begin{bmatrix}
0 & I\\
0 & -D^{-1}C
\end{bmatrix}x
+
\begin{bmatrix}
0 \\
-D^{-1}G
\end{bmatrix}
+ 
\begin{bmatrix}
0 \\
-D^{-1}H
\end{bmatrix}u.
\end{align*}
Note that $D$, $C$, and $G$ are also nonlinear functions of $x$.
Hence, the above system is an affine nonlinear system.
In the learning task, at each state-action pair $(x,u)$, the following quadratic running cost is used
\begin{align*}
r(x,u)  = x^TQx+u^2,\quad 
Q={\rm diag}\{100,\ 1000,\ 1000,\ 50,\ 10,\ 10\}.
\end{align*}
We select basis functions as  the 2nd order polynomials of  $\theta_0$,  $\theta_1$ and $\theta_2$ together with the constant function.
Hence, there are totally $10$ bases in our learning algorithm.
The initial state of the DIPC is at $(\theta_0, \theta_1, \theta_2)= (0, 10^{\circ}, -10^{\circ})$.
Here, we only consider linear controllers.
In another word, we can always write $u_t=K x_t$ for some real control gain matrix $K$.

The entire learning process is composed with  $50$ learning trials, indexed by $k=1,2,\cdots,50$.
Each learning trial is performed over a fixed time interval $[0, t_f]$.
The initial controller parameter $K_1$ is adopted from \citet{Moysis2016}:
\begin{align*}
K_1 = [-2.2361\ 499.6181\ -578.2160\ -8.2155\ 19.1832\ -88.4892].
\end{align*}
The CT-LSPE algorithm is employed in each learning trial to estimate the value function.
To facilitate the learning process, we add exploration noises in the system input in each learning trial.
As a result, the control action applied to the DIPC in the $k$-th learning trial is $u_t = K_k x_t+{\rm noise}_t$, where $\text{noise}_t$ is a stationary Gaussian process with fixed distribution  $\mathcal{N}(0,\sigma)$ for all $t\in[0, t_f]$.
Then, $u_t$ has a distribution $\mathcal{N}(K_kx_t,\sigma)\triangleq\pi_k(a_t)$ in the $k$-th trial.
Extending the actor-critic algorithm with eligibility trace in \citet{Konda2002} to the continuous-time setting, we update the control gain matrix after finishing the $k$-th  learning trial via
\begin{align*}
K_{k+1} =  K_k- hz_{t_f} V_{t_f}(x_0),
\end{align*}
where $h>0$ is the step size, $V_{t_f}$ is the value function learned in the $k$-th trial, and 
\begin{align*}
z_t = \int_0^t\partial_{K_k}\left(\log \left(\pi_k(a_s)\right)\right)ds \propto \int_0^tx_s^T\left(a_s-K_k x_s\right)ds
\end{align*}
is the eligibility trace.
Note that we only update the control gain matrix $K_k$ at the end of each learning trial in order to increase computational efficiency.

The experimental results are given in Figures \ref{fig:experiment2} and \ref{fig:dipc-paths}.
We fix $h=0.0005$ and ${t_f}=800$ during the learning process.
After $50$ learning trials, the control gain matrix becomes
\begin{align*}
K_{50} = [-2.0547\  498.9838\ -578.0119\  -16.9386\ 38.3267\  -90.5340].
\end{align*}
The cost values $V_{t_f}(x_0)$ estimated in each learning trial are plotted in Figure \ref{fig:exp2-cost}.
Clearly the cost values show a decreasing trend as the trial number increases.
We also plot the system trajectories under controllers $u_1=K_1x_t $ and $u_{50}=K_{50} x_t $ in Figure \ref{fig:exp2-path}, respectively.
Obviously, $u_{50}$ leads to a much better transient performance, in the sense that the DIPC shows less oscillations.
Finally, the weights of basis functions ($\{c_{i,t}\}_{t\geq0}$ for $i=1,2,\cdots,10$) before and after conducting the learning process are plotted in Figure \ref{fig:dipc-paths}.



\begin{figure}[t]
    \centering
    \begin{subfigure}[t]{0.35\textwidth}
        \includegraphics[width=\textwidth]{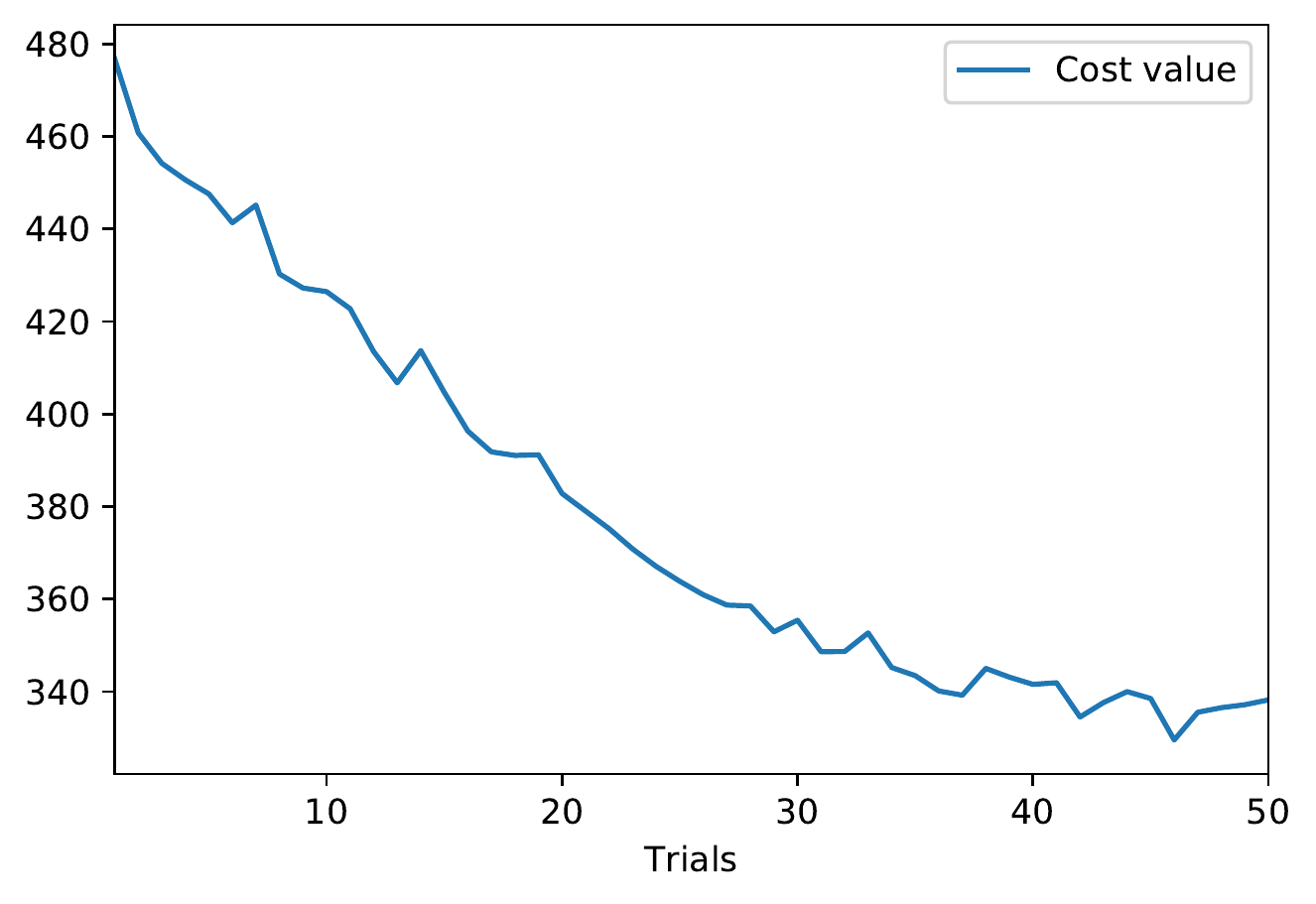}
        \caption{Costs at different learning trials.}
        \label{fig:exp2-cost}
    \end{subfigure}
    ~
        \begin{subfigure}[t]{0.62\textwidth}
        \includegraphics[width=\textwidth]{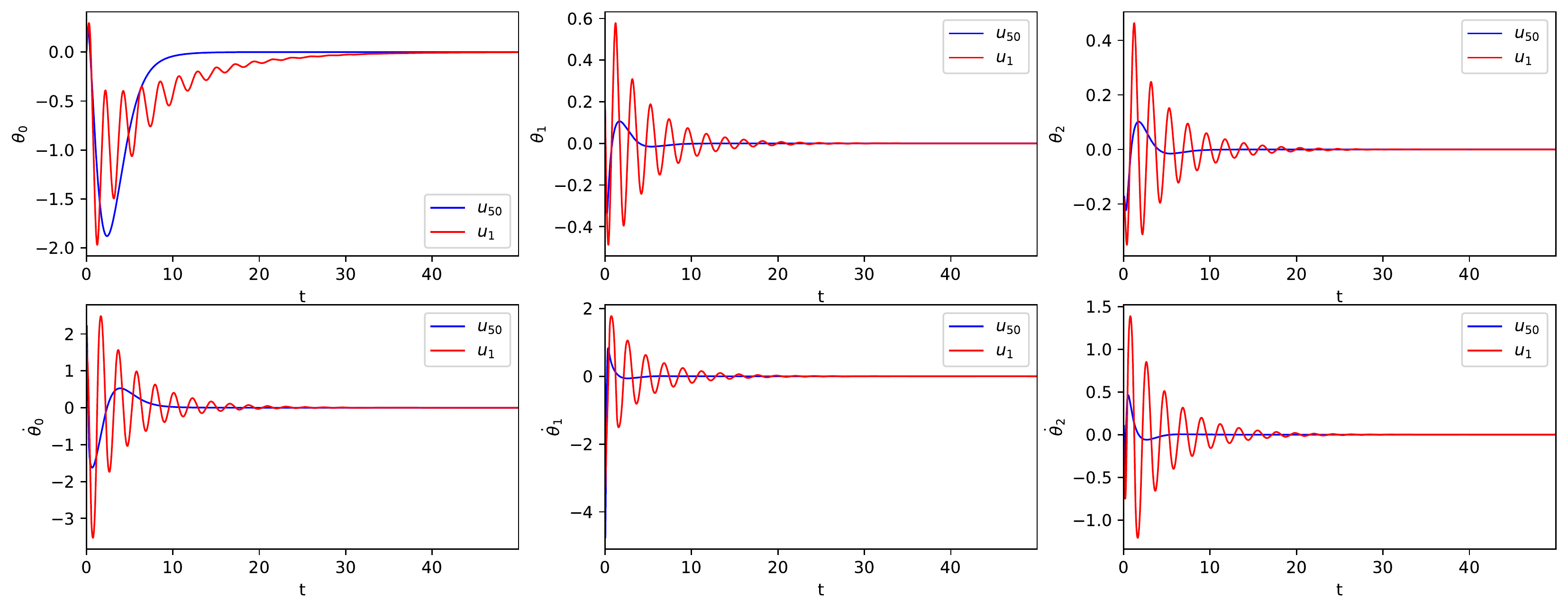}
        \caption{Comparison between different controllers.}
        \label{fig:exp2-path}
         \end{subfigure}

    \caption{The value functions and system trajectories of DIPC, before and after conducting the learning process.}\label{fig:experiment2}
\end{figure}

\begin{figure}
\centering
           \includegraphics[width=\textwidth]{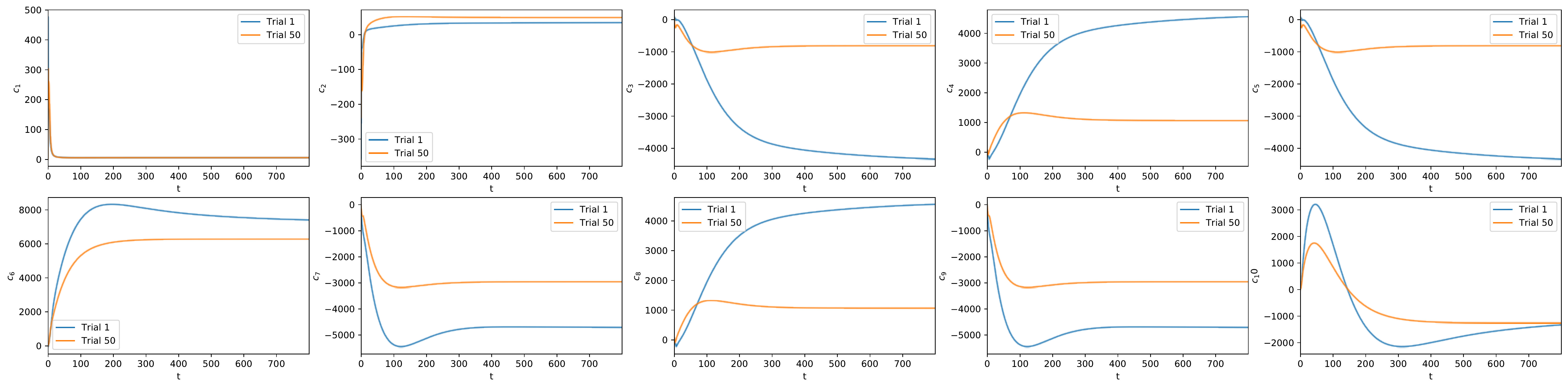}
    \caption{The weights of basis functions ($\{c_{i,t}\}_{t\geq0}$ for $i=1,2,\cdots,10$) for the DIPC, in the first and the $50$-th learning trials.}
    \label{fig:dipc-paths}
\end{figure}

\section{Conclusions}

This paper is motivated to provide a new RL framework for learning problems in continuous environments.
A novel concept, known as the temporal-differential error, and a new class of temporal-differential learning methods, are proposed.
In particular, two RL algorithms with detailed convergence analysis are designed for RL in continuous environments.  
Next, we point out several future research directions that deserve further investigations.
First, the proposed algorithms are purely based on the temporal-differential error.
It is interesting to see how to incorporate eligibility traces in RL algorithms driven by the temporal-differential error.
Second, it is worth checking how to extend other existing discrete-time TD-learning-based RL algorithms to the continuous setting, by means of the techniques developed in this paper.
Finally, besides the prediction problem which is the main focus of this paper, it is also important to investigate the off-policy learning problem in continuous environments.


\acks{The work of Z. P. Jiang has been supported partially by the National Science Foundation under Grants ECCS-1501044 and EPCN-1903781.
The opinions expressed in this paper are those of the authors and do not necessarily reflect the views and policies of Bank of America.}


\appendix

\section{Review of Continuous-time Markov Processes}\label{appendix:markov}

Let $(\Omega, \mathcal{F}, \mathbb{P})$ be a probability space. 
Consider $\{x_t\}_{t\geq0}$ as a homogeneous c{\`a}dl{\`a}g (right continuous with left limits) strong Markov process on state space $\mX$, which is a locally compact Hausdorff space with a countable base and is equipped with the Borel $\sigma$-algebra $\mB$.
The definition of $\mX$ covers most spaces that will be encountered in RL tasks, including Euclidean spaces, countable discrete spaces, and topological manifolds.
Starting with $x_0=x$, $x_t$ admits a stationary probability measure $\mu$ on $\mB$. 
Denote by $\mE_x$ the expectation conditionally on the initial state $x$ and by $\mE_\mu$ the expectation with respect to $\mu$.
$[M]_t$ denotes the quadratic variation of a real-valued stochastic process $M_t$ defined on $(\Omega, \mathcal{F}, \mathbb{P})$.

Given the stationary distribution $\mu$,
we can define a Hilbert space $L^2_\mu$ equipped with the inner product $\langle\cdot,\cdot\rangle_\mu$ and the norm $\|\cdot \|_\mu$:
\begin{align*}
\langle f,g\rangle_\mu = \int_\mX fg\ d\mu = \mE_\mu\left[ fg\right],\quad \|f\|_\mu = \sqrt{\langle f,f\rangle_\mu}, \quad \forall f,\ g\in L^2_\mu.
\end{align*}
In addition, we can define an operator  $P_t:L^2_\mu\rightarrow L^2_\mu$ as
\begin{align*}
    P_tf(x)=\mE_x[f(x_t)],\quad \forall x\in \mX.
\end{align*}
One can easily check that $P_t$ is a contraction semigroup on $L^2_\mu$ (see Lemma \ref{lem:contraction} in Section \ref{appendix:lemma-proof}).
In addition, the infinitesimal generator $\mA$ associated with $P_t$ is defined as
\begin{align*}
    \mathcal{A}f = \lim_{t\rightarrow0}\frac{P_tf-f}{t},
\end{align*}
provided that the above limit exists.
Denote the domain and the range of $\mA$ as $\mDA\subseteq L^2_\mu$ and $\mRange(\mA)\subseteq L^2_\mu$, respectively.
By Hille–Yosida theorem \citep{Pazy1983}, we know $\mA$ is closed, and  $\mDA$ is a dense subset of $L^2_\mu$.
Since $x_t$ is right continuous, its transition probability is stochastically continuous, and is uniquely determined by its  infinitesimal generator.
For additional properties on $P_t$ and $\mA$, see \citet{Nelson1958,Pazy1983,Hansen1995}.

\section{Proofs of Theorems}\label{appendix:theorem-proof}
In this section, we present proofs of our  three main theorems.
\subsection{Proof of Theorem \ref{thm:approximated-pe} }\label{appendix:proof-approximated-pe}

We first show that $V_t$ exists on $\mH$.
For any $V\in \mH$, we have from the  definition of $K$ that
\begin{align*}
    \langle V, \Pi\mA V\rangle_\mu
     & = \int_{\mX}  V(x)\int_{\mX} (\mA V)(z)K(z,x)d\mu(z)d\mu(x)\\
     & =  \int_{\mX}(\mA V)(z)\int_{\mX}  V(x)K(x,z)d\mu(x)d\mu(z)\\
     & = \langle  V, \mA V\rangle_\mu.
\end{align*}
Since $\mA$ generates a contraction semigroup $P_t$ in $L^2_\mu$ (Lemma \ref{lem:contraction}), we know by Lumer–Phillips theorem \citep{Lumer1961} that $\mA$ is closed and dissipative.
Hence, $\langle  V, \mA V\rangle_\mu\leq0$.
Thus, $\Pi\mA$ is dissipative on $\mH$, and $-\Pi\mA$ is maximal monotone \citep[Proposition 1]{Borwein2010}.
Then, we have by Minty surjectivity theorem \citep{Minty1962} that $\Pi\mA-\mID$ is surjective. 
In addition, since $\mA$ is closed, we know by the definition of $\Pi$ and the dominated convergence theorem that $\Pi\mA$ is closed on $\mH$.
Hence, again by Lumer–Phillips theorem, $\Pi\mA$ generates a contraction  semigroup on $\mH$.
Then we have from \citet{Phillips1954} that the abstract Cauchy problem \eqref{equ:approximated-pe} admits a unique solution $V_t$.
In addition, by Hille-Yosida theorem,  $(\Pi\mA-\gamma)^{-1}$ is bounded, and hence 
\begin{align*}
0= \Pi (\mT \hat V^*+r)= (\Pi\mA-\gamma) \hat V^*+\Pi r
\end{align*}
 admits a unique solution $\hat V^*$ on $\mH$.

Denote $\tilde V_t= V_t-\hat V^*$.
We next show that $V_t$ converges to $\hat V^*$ exponentially.
Note from \eqref{equ:approximated-pe} and the above analysis that
\begin{align*}
    \frac{d}{dt}\left\|\tilde V_t(x)\right\|_\mu^2
      = 2\int_{\mX} \tilde V_t(x)\int_{\mX} \left(\mA\tilde V_t-\gamma \tilde V_t\right)(z)K(z,x)d\mu(z)d\mu(x)
     \leq -2\gamma \left\|\tilde V_t\right\|_\mu^2.
\end{align*}
This implies that $V_t$ is globally exponentially stable at $\hat V^*$.

Now, to derive the first finite-time error bound, we have
\begin{align*}
    \left\| V_t-V^*\right\|_\mu
     \leq \left\| V_t-\hat V^*\right\|_\mu + \left\|\hat V^*-V^*\right\|_\mu
     \leq e^{-\gamma t}\left\| V_0-\hat V^*\right\|_\mu + \left\|\hat V^*-V^*\right\|_\mu.
\end{align*}
Denote $0=\mT (\hat V^*)+r+\varepsilon$, 
where $\varepsilon$ is the approximation error due to $\Pi$.
Using Dynkin's formula and the tower property, we have 
\begin{align*}
     \hat V^*(x)&=\mE_{x}\left[e^{-\gamma t}  \hat V^*(x_t)+\int_0^{t}e^{-\gamma s}(r(x_s)+\varepsilon(x_s))ds\right],\\
    V^*(x)&=\mE_{x}\left[e^{-\gamma t}V^*(x_t)+\int_0^{t}e^{-\gamma s}r(x_s)ds\right].
\end{align*}
Since $\mu$ is the stationary distribution, we have by Jenson's inequality and Fubini's theorem that
\begin{align*}
    \gamma^2\left\|\hat V^*-V^*\right\|_\mu^2
    =\mE_\mu\left[\left|\mE_{x}\left[\int_0^{\infty} \gamma e^{-\gamma s}\varepsilon(x_s)ds\right]\right|^2\right]
    &\leq \mE_\mu\left[\int_0^{\infty} \gamma e^{-\gamma s}\varepsilon^2(x_s)ds\right]\\
     &= \left\|\mT (\hat V^*)+r\right\|_{\mu}^2.
\end{align*}

To derive the second finite-time error bound, we rewrite \eqref{equ:approximated-pe} as
\begin{align*}
    \partial_t V_t(x)=\mT  V_t(x) + r(x)+\varepsilon_t(x),
\end{align*}
where $\varepsilon_t$ is the approximation error at time $t$.
Using Dynkin's formula, we have 
\begin{align*}
     V_t(x)&=\mE_{x}\left[e^{-\gamma t} V_0(x_t)+\int_0^{t}e^{-\gamma s}(r(x_s)+\varepsilon_{t-s}(x_s))ds\right],
\end{align*}
and hence
\begin{align*}
     V_t(x) - V^*(x)&=\mE_{x}\left[e^{-\gamma t} (V_0-V^*)(x_t)+\int_0^{t}e^{-\gamma s}\varepsilon_{t-s}(x_s)ds\right].
\end{align*}
Then, 
\begin{align}
    &\quad\frac{\gamma}{1-e^{-\gamma t}}\left\| V_t-V^*\right\|_\mu\notag\\
    &\leq\frac{\gamma e^{-\gamma t}}{1-e^{-\gamma t}}\left\|  \mE_{x}\left[(V_0-V^*)(x_t)\right]\right\|_\mu+\left\| \mE_{x}\left[\int_0^{t}\frac{\gamma e^{-\gamma s}}{1-e^{-\gamma t}}\varepsilon_{t-s}(x_s)ds\right]\right\|_\mu\notag\\
    &=\frac{\gamma e^{-\gamma t}}{1-e^{-\gamma t}}\sqrt{\mE_\mu \left[  \left|\mE_{x}\left[(V_0-V^*)(x_t)\right]\right|^2\right]}
    +\sqrt{\mE_\mu\left[\left| \mE_{x}\left[\int_0^{t}\frac{\gamma e^{-\gamma s}}{1-e^{-\gamma t}}\varepsilon_{t-s}(x_s)ds\right]\right|^2\right]}\notag\\
    &\leq\frac{\gamma e^{-\gamma t}}{1-e^{-\gamma t}}\sqrt{\mE_\mu \left[  \left|(V_0-V^*)(x_t)\right|^2\right]}
    +\sqrt{\mE_\mu\left[\left| \int_0^{t}\frac{\gamma e^{-\gamma s}}{1-e^{-\gamma t}}\varepsilon_{t-s}(x_s)ds\right|^2\right]}\label{equ:jenson-3}\\
    &\leq \frac{\gamma e^{-\gamma t}}{1-e^{-\gamma t}} \left\| V_0-V^*\right\|_\mu+\sqrt{\int_0^{t}\frac{\gamma e^{-\gamma s}}{1-e^{-\gamma t}}\|\varepsilon_{t-s}\|_\mu^2ds}\label{equ:jenson-2}\\
     &= \frac{\gamma e^{-\gamma t}}{1-e^{-\gamma t}} \left\| V_0-V^*\right\|_\mu+\sqrt{\int_0^{t}\frac{\gamma e^{-\gamma (t-s)}}{1-e^{-\gamma t}}\left\|(\Pi-\mID)(\mT V_s+r)\right\|_{\mu}^2ds}\notag\\
    &\leq \frac{\gamma e^{-\gamma t}}{1-e^{-\gamma t}} \left\| V_0-V^*\right\|_\mu+\sup_{s\in[0,t]}\left\|(\Pi-\mID)(\mT V_s+r)\right\|_{\mu},\notag
\end{align}
where
\eqref{equ:jenson-3} and \eqref{equ:jenson-2} above are due to Jenson's inequality.
This completes the proof.

\subsection{Proof of Theorem \ref{thm:lspe}}\label{appendix:proof-lspe}
We first show $R_t = \left(\int_0^t\phi_s\phi^T_sds+\rho I\right)^{-1}$. 
Since
\begin{align*}
    \frac{d}{dt}\left(R_t\left(\int_0^t\phi_s\phi^T_sds+\rho I\right)\right) = \dot R_t\left(\int_0^t\phi_s\phi^T_sds+\rho I\right)+R_t\phi_t\phi^T_t,
\end{align*}
$R_t = \left(\int_0^t\phi_s\phi^T_sds+\rho I\right)^{-1}$ solves the ODE of $R_t$ in Algorithm \ref{alg:temporal-differential}.
Since the right-hand-side of the ODE of $R_t$ is locally Lipschitz, and $\left(\int_0^t\phi_s\phi^T_sds+\rho I\right)^{-1}$ is bounded, we can select a sufficiently large subset $S$ of the space of all symmetric positive definite matrices, such that $R_t = \left(\int_0^t\phi_s\phi^T_sds+\rho I\right)^{-1}$ is a unique solution of the ODE of $R_t$ on $S$, and it also remains in $S$.

We next denote
\begin{align*}
    A_t &\triangleq \frac{1}{t+1}\int_0^t\phi_s(\mA\phi_s-\gamma \phi_s)^Tds,\quad 
    b_t\triangleq\frac{1}{t+1}\int_0^t\phi_s r_sds,\\
    \Sigma_t &\triangleq  \frac{1}{t+1}\left(\int_0^t\phi_s\phi^T_sds+\rho I\right),\quad
    \eta_t \triangleq \frac{1}{t+1}\int_0^t\phi_sdM_s^T.
\end{align*}
Then, we can rewrite the updating equation of $c_t$ in Algorithm \ref{alg:temporal-differential} as 
\begin{align}
    \dot c_t 
     = \Sigma^{-1}Ac_t+\Sigma^{-1}b +(\Sigma_t^{-1}A_t-\Sigma^{-1}A)c_t + \Sigma_t^{-1}b_t-\Sigma^{-1}b+\Sigma_t^{-1}\eta_tc_t,\label{equ:randomized-vi-empirical}
\end{align}
where
\begin{align*}
    A &\triangleq \int_{\mX}\phi(\mA\phi-\gamma \phi)^Td\mu,\quad b\triangleq\int_{\mX}\phi rd\mu,\quad
    \Sigma\triangleq \int_{\mX}\phi\phi^Td\mu.
\end{align*}



Let $\tilde c_t=c_t- c^*$.
Then,
\begin{align*}
    \dot{\tilde c}_t 
      = \Sigma^{-1}A\tilde c_t &+\left(\Sigma_t^{-1}A_t-\Sigma^{-1}A +\Sigma_t^{-1}\eta_t\right)\tilde c_t
      + \Sigma_t^{-1}b_t-\Sigma^{-1}b+\left(\Sigma_t^{-1}A_t-\Sigma^{-1}A + \Sigma_t^{-1}\eta_t\right)c^*.
\end{align*}
By Assumptions \ref{assumption:poisson}, we have
\begin{align*}
    \mA \psi_i(x_s) =\Delta_{i,s},
    \end{align*}
where $i=2,3,4$, $\psi_4= \psi_1-\gamma\psi_3$, and
\begin{align*}
    \Delta_{2,s} =\phi(x_s)\phi^T(x_s)-\Sigma,\quad
     \Delta_{3,s} =\phi(x_s) r(x_s)-b,\quad
         \Delta_{4,s} &=\phi(x_s)(\mA\phi^T(x_s)-\gamma \phi^T(x_s))-A.
\end{align*}
Then,  we have by Lemma \ref{lemma:poisson-equation} (element-wise) that
\begin{align*}
\lim_{t\rightarrow\infty}\mE_x\left[\psi_i(x_t)\right]=0, \quad  \sup_{t\geq0}\frac{1}{t+1}\mE_x\left[\left|\int_0^t\Delta_{i,s}ds\right|^2\right]<\infty,\quad i=2,3,4.
\end{align*}

Now consider the boundedness of $\eta_t$.
The quadratic variation of the $j$-th element in $M_t$ is given below: 
\begin{align*}
    d[M_j]_t & = (dM_{j,t})^2 = (d\phi_j(x_t))^2 + O(dt^2).
\end{align*}
By Assumption \ref{assumption:noise} and Ito's isometry, we have 
\begin{align*}
    \sup_{t\geq0}(t+1)\mE_x\left[\left|\eta_{t}\right|^2\right]<\infty.
\end{align*}


Applying  Lemma \ref{lemma:law-large-number} with $\alpha_t=1/\sqrt{t+1}$, we know with probability one that
$A_t$, $b_t$, and $\Sigma_t$ converge to $A$, $b$, and $\Sigma$, respectively.
In addition, we know with a probability at least of $1-C_0/\epsilon^2$ that
\begin{align*}
    |A_t-A|\leq \frac{\epsilon}{\sqrt{t+1}},\quad
    |b_t-b|\leq \frac{\epsilon}{\sqrt{t+1}}, \quad 
    |\Sigma_t-\Sigma|\leq \frac{\epsilon}{\sqrt{t+1}}, \quad \text{and}\quad 
    |\eta_t|\leq \frac{\epsilon}{\sqrt{t+1}}, 
\end{align*}
for all $t\geq0$,
where 
\begin{align*}
    C_0 = \sup_{t\geq0}\max\left\{ 
    (t+1)\mE_x\left[\left|\eta_t\right|^2\right],\ 
        \frac{1}{t+1}\mE_x\left[\left|\int_0^t\Delta_{i,s}ds\right|^2\right],\ i=2,3,4
    \right\}.
\end{align*}
As a result, by the boundedness of $A_t$, $b_t$, $\Sigma_t$, and $\eta_t$, we have with a probability at least  $1-C_0/\epsilon^2$ that
\begin{align*}
    \left|\Sigma_t^{-1}\right| & \leq {\rho_0},\quad 
    \left|I-\Sigma\Sigma_t^{-1}\right|  \leq  \left|\Sigma_t-\Sigma\right| \left|\Sigma_t^{-1}\right|\leq \frac{\epsilon\rho_0}{t+1},\\
     \left|\Sigma\Sigma_t^{-1}\right| & \leq  1+ \frac{\epsilon\rho_0}{t+1},\quad
    \left|\Sigma\Sigma_t^{-1}\eta_t\right|\leq \frac{\epsilon\left(1+ \epsilon\rho_0\right)}{\sqrt{t+1}},\\
\left|\Sigma\Sigma_t^{-1}A_t-A\right|
    & \leq \left|\Sigma\Sigma_t^{-1}\right|\left|A_t-A\right|+ \left|\Sigma\Sigma_t^{-1}-I\right|\left|A\right|
       \leq \frac{\epsilon}{\sqrt{t+1}}\left(\epsilon\rho_0+\rho_0|A|+1\right),\\
 \left|\Sigma\Sigma_t^{-1}b_t-b\right|
    & \leq \frac{\epsilon}{\sqrt{t+1}}\left( \epsilon\rho _0+\rho_0\left|b\right|+1\right),
\end{align*}
where
\begin{align*}
\rho_0 = \frac{\epsilon+|\Sigma|}{|\Sigma|^2} + \frac{\left(\epsilon+|\Sigma|\right)^2}{\rho|\Sigma|^2}.
\end{align*}
Applying Lemma \ref{lemma:mean-square-iss}, we know $c_t$ converges to $c^*$ with probability one. 
In addition,
\begin{align*}
    \left|c_t-c^*\right|^2\leq \frac{\lambda_M}{\lambda_m}e^{\zeta(t,0)}\left|c_0-c^*\right|^2+\frac{\lambda_M}{\varepsilon\lambda_m}\int_{0}^t\xi(s)e^{\zeta(t,s)}ds
\end{align*}
with a probability at least $1-C_0/\epsilon^2$, 
where $\lambda_M$ and $\lambda_m$ represent the largest and the smallest eigenvalues of $\Sigma$, respectively, $\varepsilon <2\gamma/\lambda_M$ is an arbitrary positive constant, and by Lemma \ref{lem:bounds},
\begin{align*}
    \zeta(t,s) & 
= -\left(\frac{2\gamma}{\lambda_M}-\varepsilon\right)(t-s)
+4\epsilon\left(2+ 2\epsilon\rho_0+\rho_0\left|A\right|\right)\left(\sqrt{t+1}-\sqrt{s+1}\right)\\
&\leq -2\left(\frac{\gamma}{\lambda_M}-\varepsilon\right)(t-s)+ C_1,\\
\xi(s) & 
= \frac{ C_2}{s+1},
\end{align*}
where
\begin{align*}
C_1 = \varepsilon+\frac{4\epsilon^2\left(2+ 2\epsilon\rho_0+\rho_0\left|A\right|\right)^2}{\varepsilon},\quad 
C_2 = \epsilon^2\left(\rho_0\left(\left|b\right|+\left|A\right||c^*|\right)+(1+\epsilon\rho_0)(1+2|c^*|)\right)^2.
\end{align*}
In particular,  we know by Lemma \ref{lem:int-bound} that $\left|c_t-c^*\right|^2=O\left(t^{-1}\right)$.
This completes the proof.

\subsection{Proof of Theorem \ref{thm:td}}\label{appendix:proof-td}
First, we introduce the following auxillary system:
\begin{align*}
    \dot{\hat c}_t = \alpha_t(A\hat c_t+b),\quad \hat c_0 = c_0,
\end{align*}
where $A$ and $b$ follow the definitions in Appendix \ref{appendix:proof-lspe}.

We first show that $\hat c_t$ converges to $c^*$.
Indeed, a direct calculation shows that 
\begin{align*}
    \frac{d}{dt}\left|\hat c_t-c^*\right|^2  = 2\alpha_t(\hat c_t-c^*)^TA(\hat c_t-c^*) \leq -2\gamma\alpha_t|\hat c_t-c^*|^2.
\end{align*}
By the comparison lemma \citep[Lemma C.3.1]{Sontag1998}, we have 
\begin{align*}
    |\hat c_t-c^*|^2 \leq e^{-2\gamma\int_{0}^t \alpha_sds}|\hat c_0-c^*|^2.
\end{align*}
Since $\int_0^\infty\alpha_tdt = \infty$,  $\hat c_t$ converges to $c^*$ as $t$ goes to the infinity.

Denote $\tilde c = c-\hat c$.
We can rewrite the updating equation in Algorithm \ref{alg:temporal-differential-infty} as 
\begin{align*}
    d\tilde c_t = \alpha_t(A\tilde c_t+\Delta_{4,t}\tilde c_t+\Delta_{4,t}\hat c_t+\Delta_{3,t})dt + \alpha_t\phi(x_t)(\tilde c_t+\hat c_t)^TdM_t,\quad \tilde c_0=0,
\end{align*}
where $\Delta_{3,t}$ and $\Delta_{4,t}$ follow the definitions in Appendix \ref{appendix:proof-lspe}.

Note that $c^TdM = \sum_jc_jdM_j$, where $dM_j$ denotes the $j$-th element in $dM$.
Now the proof is completed by applying  Lemma \ref{lemma:mean-square-iss-alpha}.



            

\section{Technical Lemmas}\label{appendix:lemma-proof}
In this section, we present several supporting lemmas that have been used in the proofs of our main results.
\begin{lemma}\label{lem:contraction}
    $P_t$ is a contraction semigroup in $L^2_\mu$.
\end{lemma}
\begin{proof}
    We have by Jenson's inequality that 
    \begin{align*}
        \|P_tf\|_\mu^2 = \int_{\mX}|P_tf(x)|^2d\mu(x) \leq \int_{\mX}P_t(f^2)(x)d\mu(x) =\int_{\mX}f^2(x)d\mu(x) = \|f\|_\mu^2,\quad \forall f\in L^2_\mu.
    \end{align*}
This completes the proof.
\end{proof}

\begin{lemma}\label{lemma:law-large-number}
Suppose $S_t$ is a time series satisfying 
    \begin{align*}
     \alpha_t^2 \mathbb{E}\left[S_{t}^{2}\right] \leq C,\quad \forall t\geq0,
    \end{align*}
    for some $C<\infty$, where $\alpha_t>0$ is deterministic and decreases monotonically  to $0$. 
    Then, for any $\epsilon>\sqrt{C}$, 
    \begin{align*}
    \mathbb{P}\left(\sup _{t\geq0}\left|\alpha_t S_{t}\right|<\epsilon\right) >1-\frac{1}{\epsilon^{2}} C.
    \end{align*}
    In particular, $\alpha_t^2 S_{t}=O({\alpha_t})$ with probability one.
\end{lemma}

\begin{proof}
    Denote events
    \begin{align*}
        E_{T}=\left\{\max _{0\leq t \leq T}\left|{\alpha_t} S_{t}\right|>\epsilon\right\},\quad \forall T\geq0.
    \end{align*}
    By Markov's inequality, $\mathbb{P}\left(E_{T}\right) \leq {C}/{\epsilon^{2}}$. 
    Since $E_{T}$ is increasing, by monotone convergence theorem,
    \begin{align*}
        \lim _{T \rightarrow \infty} \mathbb{P}\left(E_{T}\right)=\mathbb{P}\left(\lim _{T \rightarrow \infty} E_{T}\right)=\mathbb{P}\left(\sup _{t\geq0}\left|{\alpha_t} S_{t}\right|>\epsilon\right) \leq \frac{1}{\epsilon^{2}} C.
    \end{align*}
    Letting $\epsilon$ go to the infinity, we have
    \begin{align*}
\mathbb{P}\left(\sup _{t\geq0}\left|{\alpha_t} S_{t}\right|<\infty\right)=1.
\end{align*}
This completes the proof.
\end{proof}

\begin{lemma}\label{lemma:poisson-equation}
Suppose Assumption \ref{assumption:ergodic} holds.
Denote by $\psi$ the solution to the following Poisson equation:
\begin{align*}
    \mA\psi(x_t) =f(x_t)-\mE_\mu[f].
\end{align*}
Then, 
\begin{align*}
\lim_{t\rightarrow\infty}\mE_x\left[\psi(x_t)\right]=0,\quad \limsup_{t\rightarrow\infty}\frac{1}{t}\mE_x\left[\left|\int_0^t \mA\psi(x_t) dt\right|^2\right]<\infty.
\end{align*}
\end{lemma}

\begin{proof}
Since $\psi\in\mDA$, $\mE_\mu\left[\psi^2\right]<\infty$ and $\mE_\mu\left[f^2\right]<\infty$.
By Assumption \ref{assumption:ergodic} and the ergodic theorems \citep[Theorem 2.2 and Corollary 2.3]{Kontoyiannis2003}, 
\begin{align*}
\lim_{t\rightarrow\infty}\mE_x\left[\psi(x_t)\right]=\mE_\mu[\psi]=0,\quad
\lim_{t\rightarrow\infty}\mE_x\left[f^2(x_t)\right]=\mE_\mu\left[f^2\right].
\end{align*}
Denote 
\begin{align*}
    m_t = \psi(x_t) - \psi(x) +\int_0^t \mA \psi(x_s)ds.
\end{align*}
Then $m$ is a  martingale, and we have by the martingale property that
\begin{align*}
    \mE_x[m_t^2] = \mE_x\left[\left(\sum_i\Delta m_{i,t}\right)^2\right] =  \sum_i\mE_x[(\Delta m_{i,t})^2] =  \sum_i\mE_x\left[\left(\frac{\Delta m_{i,t}}{\sqrt{\Delta t_i}}\right)^2\Delta t_i\right],
\end{align*}
where $0=t_1<t_2<\cdots =t$ is a set of partition points on the interval $[0,t]$, $\Delta t_i = t_{i+1}-t_i$, and
\begin{align*}
    \Delta m_{i,t} = \psi(x_{t_{i+1}}) - \psi(x_{t_i}) +\int_{t_i}^{t_{i+1}} \mA \psi(x_s)ds.
\end{align*}
We next quantify $\mE_x[m_t^2]$ by inspecting $\mE_x[\Delta m_{i,t}^2]$.
By H{\"o}lder's inequality,
\begin{align*}
    \mE_x\left[\left(\frac{1}{\sqrt{\Delta t_i}}\int_{t_i}^{t_{i+1}} \mA \psi(x_s)ds\right)^2\right] \leq \int_{t_i}^{t_{i+1}} \mE_x\left[\left(\mA \psi(x_s)\right)^2\right]ds =O(\Delta t_i), 
\end{align*}
where the last equality holds since the integrand is bounded.
In addition,
\begin{align*}
    &\quad\mE_x\left[\left(\psi(x_{t_{i+1}})  - \psi(x_{t_i})\right)^2\right]\\
& =\mE_x\left[\psi^2(x_{t_{i+1}}) - 2\psi(x_{t_{i+1}})\psi(x_{t_i}) + \psi^2(x_{t_i})\right]\\
& =\mE_x\left[\psi^2(x_{t_{i+1}})\right] + \mE_x\left[\psi^2(x_{t_i})\right]- 2\mE_x\left[P_{\Delta t_i}\psi(x_{t_{i}})\psi(x_{t_i})\right]\\ 
& =\mE_x\left[\psi^2(x_{t_{i+1}})\right] - \mE_x\left[\psi^2(x_{t_i})\right]- 2\mE_x\left[(P_{\Delta t_i}-\mID)\psi(x_{t_{i}})\psi(x_{t_i})\right]\\ 
& =\mE_x\left[\psi^2(x_{t_{i+1}})\right] - \mE_x\left[\psi^2(x_{t_i})\right]- 2\Delta t_i\mE_x\left[\frac{(P_{\Delta t_i}-\mID)\psi(x_{t_{i}})}{\Delta t_i}\psi(x_{t_i})\right].
\end{align*}
Putting the above items back into the summation in $\mE_x[m_t^2]$ and letting $\Delta t_i$ go to $0$, we have
\begin{align*}
    \mE_x[m_t^2] = \mE_x\left[\psi^2(x_t)\right] - \psi^2(x)  - 2\int_0^t\mE_x\left[\mA\psi(x_s)\psi(x_s)\right]ds. 
\end{align*}
Since both $\psi$ and $f$ are square integrable, $\limsup_{t\rightarrow\infty}t^{-1}\mE_x[m_t^2]<\infty$.
This completes the proof.
\end{proof}

\begin{lemma}\label{lemma:mean-square-iss}
    Consider the following system:
    \begin{align*}
        \dot c_t=(\Sigma^{-1}A+\Delta A_{t})c_t+\Delta_{t},
    \end{align*}
    where $\Sigma$ is a symmetric positive definite matrix, $A$ is a negative definite real matrix, and $\Delta A_{t}$ and $\Delta_{t}$ are random variables bounded with probability one. 
    Assume 
    \begin{align*}
        \lim_{t\rightarrow\infty}\frac{1}{t}\int_0^t\left|\Delta A_{s}\right|ds = 0
    \end{align*}
    for all $\omega\in E\in\mathcal{F}$.
Then $c$ is input-to-state stable (ISS) \citep{Sontag2008} over $E$, from $\Delta_{t}$ to $c_t$.
\end{lemma}
\begin{proof}
    Consider the Lyapunov function $V(c) = c^T\Sigma c$.
    Given $\omega\in E$,
    \begin{align*}
        \dot V(c_t)
         & = -c_t^T(A+A^T)c_t+2c_t^T\Sigma\Delta A_tc_t+2c_t^T\Sigma\Delta_{t}\\
         & \leq -(\gamma-\varepsilon|\Sigma|-2|\Sigma||\Delta A_{t}|) |c_t|^2+  \frac{1}{\varepsilon}\Delta_{t}^T\Sigma\Delta_{t}\\
         & \leq -(\gamma|\Sigma|^{-1}-\varepsilon-2|\Delta A_{t}|) V(c_t) + \frac{1}{\varepsilon}\Delta_{t}^T\Sigma\Delta_{t},
    \end{align*}
    where $\gamma>0$ satisfies  $A^T+A\leq -\gamma I$, and $\varepsilon<\gamma|\Sigma|^{-1}$ is a positive constant.
    Reformulating the above inequality, we have
    \begin{align*}
    \frac{d}{dt} \left(V(c_t)e^{(\gamma|\Sigma|^{-1}-\varepsilon)t-2\int_0^t|\Delta A_{s}|ds}\right)
     \leq  \frac{|\Sigma|}{\varepsilon}|\Delta_{t}|^2e^{(\gamma|\Sigma|^{-1}-\varepsilon)t-2\int_0^t|\Delta A_{s}|ds}.
\end{align*}
Hence,
\begin{align*}
    \left|c_t\right|^2\leq \frac{\lambda_M}{\lambda_m}e^{-(|\gamma/\lambda_M-\varepsilon)t+2\int_0^t|\Delta A_{s}|ds}\left|c_0\right|^2+\frac{\lambda_M}{\varepsilon \lambda_m}\int_0^te^{-(\gamma/\lambda_M-\varepsilon)(t-s)+2\int_s^t|\Delta A_{\tau}|d\tau}\left|\Delta_{s}\right|^2ds.
\end{align*}
where $\lambda_M$ and $\lambda_m$ represent the largest and the smallest eigenvalues of $\Sigma$, respectively.
    One can see from the above inequality and our assumption on $\Delta A$ that $c_t$ is  ISS, from $\Delta_{t}$ to $c_t$.
This completes the proof. 
\end{proof}

\begin{lemma}\label{lemma:mean-square-iss-alpha}
    Consider the following system:
    \begin{align}
        dc_t=\alpha_t(Ac_t+\Delta_{1,t}c_t+\Delta_{2,t})dt + \alpha_t\sum_{j=1}^Ng_{j,t}(c_t)dM_{j,t},\label{equ:lemma10-equ1}
    \end{align}
    where $A$ is a real matrix satisfying $A+A^T<-2\gamma I$ for some $\gamma>0$, $\alpha_t>0$ is continuously differentiable with $\dot\alpha_t\leq0$, $\lim_{t\rightarrow\infty}\alpha_t=0$ and $\int_0^\infty \alpha_t=\infty$, $g_j$ is a vector-valued function,  $M_j$ is a real-valued martingale,  and $\Delta_{1,t}$ and $\Delta_{2,t}$ are random variables  that are bounded with probability one and satisfy
    \begin{align*}
        d\psi_{i,t}=\Delta_{i,t} dt+d m_{i,t},\quad \mE_\mu\left[\Delta_{i,t}\right]=0,
    \end{align*}
for some martingales $m_i$,  $i=1,2$.
In addition, $M_j$ and $m_i$ are adapted to a filtration $\{\mathcal{F}_t\}_{t\geq0}$, 
and there exist a function $\varphi$ and a constant $C>0$, such that $d[M_j]_t\leq \varphi_t dt$ and  $d[m_{i,k}]_t\leq \varphi_tdt$ for the $k$-th elements in $m_i$, and
\begin{align*}
\mE_0\left[|\Delta_{i,t}|^4\right]\leq C, \quad 
\mE_0\left[|\psi_{i,t}|^4\right]\leq C, \quad
\mE_0\left[\varphi_{t}^4\right]\leq C, \quad
\mE_0\left[\left|\frac{g_{j,t}(c)}{1+|c|}\right|^8\right]\leq C,
\end{align*}
for all $c$, where $\mE_0$ denotes the  expectation conditional on $\mathcal{F}_0$.

Given the above conditions, we have  
\begin{align*}
\mathbb{P}\left(\left|c_{t}\right|^{2}<\epsilon\left(\left|c_{0}\right|^2+\sqrt{A_nC_0}b^2e^{2\lambda (n+1)}\right)e^{-2 \gamma \int_{0}^{t} \alpha_{s}ds}\right)>1-A_nC_0-\frac{2}{\epsilon},
\end{align*}
where $b=\max\left\{1, |c_0|\right\}$,
\begin{align*}
    A_n &= \alpha_n^2 + \left(\int_n^{\infty}\alpha_t^2dt\right)^2+ \int_n^{\infty}\alpha_t^2dt,\\
      \lambda&=\frac{\alpha_0\left(b^2+1\right)^{1/\epsilon}}{2\left(b^2+1\right)^{1/\epsilon}-2}\left(2\sqrt[4]{C}+(2\alpha_0N^2+1)\sqrt{C}+1\right)+\frac{\log2}{2},
\end{align*}
  $C_0$ is a positive constant, and $n$ and $\epsilon$ are sufficiently large integers and reals, respectively.

\end{lemma}
\begin{proof}
    The proof contains three parts. 
    First, we show $|c_t|$ cannot grow faster than the exponential rate.
    Second, we show $|c_t|$ is bounded.
    Finally, we show $|c_t|$ is asymptotically stable at the origin.

First, given $\lambda>0$ and $b=\max\{1, |c_0|\}$, we define $B_b=\left\{\omega \in \Omega : |c_t(\omega)|<be^{\lambda t}\ \text{for all $t$}\right\}$ and 
    \begin{align*}
        \sigma_{n}(\omega)=\inf \left\{t :\left|c_{t}(\omega)\right|=be^{\lambda n}\right\}, \quad \tau_{n}(\omega)=\sup \left\{t<\sigma_{n+1}(\omega) :\left|c_{t}(\omega)\right|=be^{\lambda n}\right\},\quad n\geq1.
    \end{align*}
Then, $\sigma_{n} \leq \tau_{n}<\sigma_{n+1}$.
Obviously, for any $\omega\in B_b$, we have $\sigma_{n}(\omega)\geq n$ for all $n$. 
Denote $E_{n} = \{\omega\in \Omega:\sigma_{n}(\omega)<\infty\}$, $E_{n,b} = \{\omega\in B_b:\sigma_{n}(\omega)<\infty\}$, and $\mE_0^{E}[Y]=\mE_0\left[Y \cdot 1_{\omega \in E}\right]$ for any  $E\in\mathcal{F}$ and randome variable $Y$.
Note that for any measurable function $f\geq0$,  $\mathbb{E}_{0}^{E}[f]=\mathbb{E}_{0}\left[f \cdot 1_{\omega \in E}\right]\leq\mathbb{E}_{0}\left[f\right]$.

    Denote $z_t = e^{-\lambda t}c_t$.
    Then,
    \begin{align*}
            dz_t=-\lambda z_t + \alpha_t(Az_t+\Delta_{1,t}z_t+e^{-\lambda t}\Delta_{2,t})dt + \alpha_te^{-\lambda t}\sum_{j=1}^Ng_{j,t}(c_t)dM_{j,t},
    \end{align*}
    Define $V(z)=\log\left(|z|^2+1\right)$.
    Then the Hessian of $V$ is 
    \begin{align*}
        \frac{2(|z|^2+1)I-4zz^T}{(|z|^2+1)^2}\leq \frac{2}{|z|^2+1}I\leq \frac{4}{(|z|+1)^2}I,
    \end{align*}
    and we have for any $|z_t|\geq\delta$, where $0<\delta\leq b$, that 
    \begin{align*}
        \mathcal{A} V(z_t)  
           &\leq -\frac{2\lambda\delta^2}{1+\delta^2} + \bar \Delta_t.
    \end{align*}
    where 
    \begin{align*}
           \bar \Delta_t &=  \alpha_t\left(2|\Delta_{1,t}|+1+|\Delta_{2,t}|^2\right) + 2\alpha_t^2N\sum_{j=1}^N \frac{|g_{j,t}(c_t)|^2\varphi_t}{(|c_t|+1)^2}.
    \end{align*}
    By Dynkin's formula, 
     whenever $|z_\sigma|\geq\delta$ for some $\sigma\geq0$, one has
\begin{align*}
    \mE_\sigma[V(z_{\min(t,\tau_\sigma)})]
    \leq V(z_\sigma) + \mE_\sigma\int_\sigma^{\min(t,\tau_\sigma)}\left(-\frac{2\lambda\delta^2}{1+\delta^2}+\bar \Delta_s\right)ds,
    \end{align*}
    where $\tau_\sigma$ is the first time after $\sigma$ at which $z_{\tau_\sigma}$ exists $\{z:|z|\geq\delta\}$.
    By the definition of $\lambda$,  
     $\sup_t\mE_0[\bar \Delta_t ]<{2\lambda\delta^2}(1+\delta^2)^{-1}$.
Now, we have from  the tower property that $\sup_t\mE_0\left[V(z_t)\right]\leq V(z_0)$.
Thus, by Markov's inequality, one has
\begin{align*}
    \mathbb{P}\left(B_b\right)=\mathbb{P}\left(V(z_t)<\log(b^2+1)\right)>1-\frac{\log\left(\delta^2+1\right)}{\log(b^2+1)},\quad \forall t\geq0.
\end{align*}
Hence, $z_t$ is bounded with probability one.

Now, we show $|c_t|$ is bounded. 
On one hand,  we have on $E_{n+1}$ that
\begin{align}
    \left|c_{\sigma_{n+1}}\right|^2-\left|c_{\tau_n}\right|^{2} = b^2e^{2\lambda(n+1)}-b^2e^{2\lambda n} \geq b^2e^{2\lambda n}>0. \label{equ:proof-td-equ1}
\end{align}
On the other hand, we can derive that
\begin{align*}
&\quad d\left(\frac{1}{2} |c_{t}|^2  - \alpha_tc_t^T(\psi_{1,t}c_t+\psi_{2,t})\right) \\
    &= c_t^Tdc_t + \frac{1}{2}dc_t^Tdc_t   -d\alpha_tc_t^T(\psi_{1,t}c_t+\psi_{2,t}) -\alpha_tdc_t^T(\psi_{1,t}c_t+\psi_{2,t}) \\
    &\quad -\alpha_tc_t^T(\psi_{1,t}dc_t+d\psi_{1,t}c_t+d\psi_{2,t})-\alpha_tdc_t^T(\psi_{1,t}dc_t+d\psi_{1,t}c_t+d\psi_{2,t}) -\alpha_tc_t^Td\psi_{1,t}dc_t \\
  &\leq \alpha_tc_t^T(Ac_t+\Delta_{1,t}c_t
    +\Delta_{2,t})dt+\alpha_tc_t^T\sum_{j=1}^Ng_{j,t}(c_t)dM_{j,t} 
    + \alpha_t^2\frac{N}{2}\sum_{j=1}^N |g_{j,t}(c_t)|^2\varphi_tdt \\ 
&\quad-\dot \alpha_tc_t^T(\psi_{1,t}c_t+\psi_{2,t})dt - \alpha_tc_t^T(d \psi_{1,t}c_t+d\psi_{2,t})\\
    &\quad -\alpha_t(2(\psi_{1,t}+d\psi_{1,t})c_t+(\psi_{2,t}+d\psi_{2,t}))^Tdc_t -\alpha_tdc_t^T\psi_{1,t}dc_t\\
&\leq -\alpha_tc_t^T(dm_{1,t}c_t+dm_{2,t}) + \alpha_tc_t^T\sum_{j=1}^Ng_{j,t}(c_t)dM_{j,t} + \alpha_t^2\frac{N}{2}\sum_{j=1}^N |g_{j,t}(c_t)|^2\varphi_tdt \\
    &\quad -\dot \alpha_tc_t^T(\psi_{1,t}c_t+\psi_{2,t})dt 
    -\alpha_t^2(2\psi_{1,t}c_t+\psi_{2,t})^T(Ac_t+\Delta_{1,t}c_t+\Delta_{2,t})dt \\
    &\quad-\alpha_t^2(2\psi_{1,t}c_t+\psi_{2,t})^T\sum_{j=1}^Ng_{j,t}(c_t)dM_{j,t}
    +\alpha_t^3N|\psi_{1,t}|\sum_{j=1}^N |g_{j,t}(c_t)|^2\varphi_tdt\\
        &\quad  +\alpha_t^2|c_t|^2\sum_{k=1}^{N^2}\left(dm_{1,k,t}\right)^2
        +\alpha_t^2\frac{3N}{2}\sum_{j=1}^N |g_{j,t}(c_t)|^2\varphi_tdt
        +\frac{1}{2}\alpha_t^2|d\psi_{2,t}|^2,
\end{align*}
where $m_{1,k}$ is the $k$-th element in  $m_{1}$.
Taking integration on both sides of the above inequality, we have for any $\omega\in E_{n+1}$ that
\begin{align*}
0&< \frac{1}{2} \left( |c_{\sigma_{n+1}}|^2 - |c_{\tau_n}|^2 \right) \\
&\leq  \alpha_{\sigma_{n+1}}c_{\sigma_{n+1}}^T(\psi_{1,{\sigma_{n+1}}}c_{\sigma_{n+1}}+\psi_{2,{\sigma_{n+1}}}) 
        -\alpha_{\tau_n}c_{\tau_n}^T(\psi_{1,{\tau_n}}c_{\tau_n}+\psi_{2,{\tau_n}})
        -\int_{\tau_n}^{\sigma_{n+1}}\alpha_tc_t^Tdm_{1,t}c_t\\
      & \quad  -\int_{\tau_n}^{\sigma_{n+1}}\alpha_tc_t^Tdm_{2,t} 
            + N\sum_{j=1}^N\int_{\tau_n}^{\sigma_{n+1}}\alpha_t^2\left(2+\alpha_t|\psi_{1,t}|\right) |g_{j,t}(c_t)|^2\varphi_tdt \\
    &\quad -\int_{\tau_n}^{\sigma_{n+1}}\dot \alpha_tc_t^T(\psi_{2,t}+\psi_{1,t}c_t)dt
               -2\int_{\tau_n}^{\sigma_{n+1}}\alpha_t^2c_t^T\psi_{1,t}^T(A+\Delta_{1,t})c_tdt\\
    &\quad -\int_{\tau_n}^{\sigma_{n+1}}\alpha_t^2c_t^T(2\psi_{1,t}^T\Delta_{2,t}+(A+\Delta_{1,t})^T\psi_{2,t})dt
    -\int_{\tau_n}^{\sigma_{n+1}}\alpha_t^2\psi_{2,t}^T\Delta_{2,t}dt\\
        &\quad +\sum_{j=1}^N\int_{\tau_n}^{\sigma_{n+1}}\alpha_tc_t^T (1-2\alpha_t\psi_{1,t}^T)g_{j,t}(c_t)dM_{j,t}
    -\sum_{j=1}^N\int_{\tau_n}^{\sigma_{n+1}}\alpha_t^2\psi_{2,t}^Tg_{j,t}(c_t)dM_{j,t}\\
    &\quad+\sum_{k=1}^{N^2}\int_{\tau_n}^{\sigma_{n+1}}\alpha_t^2|c_t|^2\left(dm_{1,k,t}\right)^2
        +\frac{1}{2}\sum_{k=1}^{N}\int_{\tau_n}^{\sigma_{n+1}}\alpha_t^2\left(dm_{2,k,t}\right)^2.
\end{align*}
Note that for all $\omega\in E_{n+1,b}$, $n\leq \tau_n<\sigma_{n+1}<\infty$ and $|c_t|<be^{\lambda(n+1)}$ for $\tau_n\leq t<\sigma_{n+1}$. 
Then, we have by Young's inequality, H{\"o}lder's inequality, and the definitions of $\alpha$, $\sigma_n$, and $\tau_n$ that
\begin{align*}
&\quad\frac{1}{56} \left| |c_{\sigma_{n+1}}|^2 - |c_{\tau_n}|^2 \right|^2 \\
&\leq  \alpha_{\sigma_{n+1}}^2\left|c_{\sigma_{n+1}}^T(\psi_{1,{\sigma_{n+1}}}c_{\sigma_{n+1}}+\psi_{2,{\sigma_{n+1}}})\right|^2 
        +\alpha_{\tau_n}^2\left|c_{\tau_n}^T(\psi_{1,{\tau_n}}c_{\tau_n}+\psi_{2,{\tau_n}})\right|^2\\
      & \quad  +\left|\int_{\tau_n}^{\sigma_{n+1}}\alpha_tc_t^Tdm_{1,t}c_t\right|^2
                 +\left|\int_{\tau_n}^{\sigma_{n+1}}\alpha_tc_t^Tdm_{2,t} \right|^2\\
       &\quad + N^3\sum_{j=1}^N\left|\int_{\tau_n}^{\sigma_{n+1}}\alpha_t^2\left(2+\alpha_t|\psi_{1,t}|\right) |g_{j,t}(c_t)|^2\varphi_tdt \right|^2
            +\left|\int_{\tau_n}^{\sigma_{n+1}}\dot \alpha_tc_t^T\psi_{2,t}dt \right|^2\\
    &\quad +\left|\int_{\tau_n}^{\sigma_{n+1}}\dot \alpha_tc_t^T\psi_{1,t}c_tdt\right|^2
    +4\left|\int_{\tau_n}^{\sigma_{n+1}}\alpha_t^2c_t^T\psi_{1,t}^T(A+\Delta_{1,t})c_tdt\right|^2\\
            &\quad  + \left|\int_{\tau_n}^{\sigma_{n+1}}\alpha_t^2c_t^T\left(2\psi_{1,t}^T\Delta_{2,t}+(A+\Delta_{1,t})^T\psi_{2,t}\right)dt\right|^2
      +\left|\int_{\tau_n}^{\sigma_{n+1}}\alpha_t^2\psi_{2,t}^T\Delta_{2,t}dt\right|^2\\
    &\quad +N\sum_{j=1}^N\left|\int_{\tau_n}^{\sigma_{n+1}}\alpha_tc_t^T \left(1-2\alpha_t\psi_{1,t}^T\right)g_{j,t}(c_t)dM_{j,t}\right|^2
    +N\sum_{j=1}^N\left|\int_{\tau_n}^{\sigma_{n+1}}\alpha_t^2\psi_{2,t}^Tg_{j,t}(c_t)dM_{j,t}\right|^2\\
        &\quad +N^2\sum_{k=1}^{N^2}\left|\int_{\tau_n}^{\sigma_{n+1}}\alpha_t^2|c_t|^2\left(dm_{1,k,t}\right)^2\right|^2
           +\frac{N}{4}\sum_{k=1}^{N}\left|\int_{\tau_n}^{\sigma_{n+1}}\alpha_t^2\left(dm_{2,k,t}\right)^2\right|^2\\
&\leq  2\alpha_{\sigma_{n+1}}^2\left(|\psi_{1,{\sigma_{n+1}}}|^2 b^4e^{4\lambda (n+1)}+|\psi_{2,{\sigma_{n+1}}}|^2 b^2e^{2\lambda (n+1)}\right) \\
      &\quad  +2\alpha_{\tau_n}^2\left(|\psi_{1,{\tau_n}}|^2 b^4e^{4\lambda (n+1)}+|\psi_{2,{\tau_n}}|^2b^2e^{2\lambda (n+1)}\right)
      +\left|\int_{\tau_n}^{\sigma_{n+1}}\alpha_tc_t^Tdm_{1,t}c_t\right|^2 \\
      & \quad  +\left|\int_{\tau_n}^{\sigma_{n+1}}\alpha_tc_t^Tdm_{2,t} \right|^2
        + 2N^3\sum_{j=1}^N\int_{\tau_n}^{\sigma_{n+1}}\alpha_t^2 dt \int_{\tau_n}^{\sigma_{n+1}}\alpha_t^2\left(4+\alpha_t^2|\psi_{1,t}|^2\right) |g_{j,t}(c_t)|^4\varphi_t^2dt \\
      &\quad   +\int_{\tau_n}^{\sigma_{n+1}}\dot \alpha_tdt\int_{\tau_n}^{\sigma_{n+1}}\dot \alpha_t|\psi_{2,t}|^2dt b^2e^{2\lambda (n+1)}
      +\int_{\tau_n}^{\sigma_{n+1}}\dot \alpha_tdt\int_{\tau_n}^{\sigma_{n+1}}\dot \alpha_t  |\psi_{1,t}|^2dtb^4e^{4\lambda (n+1)}\\
    &\quad  +4\int_{\tau_n}^{\sigma_{n+1}}\alpha_t^2dt\int_{\tau_n}^{\sigma_{n+1}}\alpha_t^2\left|\psi_{1,t}^T(A+\Delta_{1,t})\right|^2 dtb^4e^{4\lambda (n+1)}\\
            &\quad  +\int_{\tau_n}^{\sigma_{n+1}}\alpha_t^2dt \int_{\tau_n}^{\sigma_{n+1}}\alpha_t^2\left|2\psi_{1,t}^T\Delta_{2,t}+(A+\Delta_{1,t})^T\psi_{2,t}\right|^2dtb^2e^{2\lambda (n+1)}\\
    &\quad +\int_{\tau_n}^{\sigma_{n+1}}\alpha_t^2dt\int_{\tau_n}^{\sigma_{n+1}}\alpha_t^2\left|\psi_{2,t}^T\Delta_{2,t}\right|^2dt
    +N\sum_{j=1}^N\left|\int_{\tau_n}^{\sigma_{n+1}}\alpha_tc_t^T \left(1-2\alpha_t\psi_{1,t}^T\right)g_{j,t}(c_t)dM_{j,t}\right|^2\\
              &\quad +N\sum_{j=1}^N\left|\int_{\tau_n}^{\sigma_{n+1}}\alpha_t^2\psi_{2,t}^Tg_{j,t}(c_t)dM_{j,t}\right|^2
              +N^2\sum_{k=1}^{N^2}\left|\int_{\tau_n}^{\sigma_{n+1}}\alpha_t^2\left(dm_{1,k,t}\right)^2\right|^2b^4e^{4\lambda (n+1)}\\
           &\quad+\frac{N}{4}\sum_{k=1}^{N}\left|\int_{\tau_n}^{\sigma_{n+1}}\alpha_t^2\left(dm_{2,k,t}\right)^2\right|^2.
\end{align*}
Taking expectations on both sides of the above inequality, we have by Ito's isometry and Fubini's theorem that
\begin{align}
&\quad\frac{1}{56} \mE_{0}^{E_{n+1,b}}\left[\left| |c_{\sigma_{n+1}}|^2 - |c_{\tau_n}|^2 \right|^2\right] \notag\\
&\leq  2\alpha_n^2\sup_{t\geq0}\mE_{0}\left[|\psi_{1,t}|^2+|\psi_{2,t}|^2\right] b^4e^{4\lambda (n+1)}
+2\alpha_n^2\sup_{t\geq0}\mE_{0}\left[|\psi_{1,t}|^2 +|\psi_{2,t}|^2\right] b^4e^{4\lambda (n+1)} \notag\\
& \quad +N^2\int_{n}^{\infty}\alpha_t^2dt \mE_{0}\left[\varphi_t^2\right]b^4e^{4\lambda (n+1)}
        + N\int_{n}^{\infty}\alpha_t^2dt \mE_{0}\left[\varphi_t^2\right] b^2e^{2\lambda (n+1)} \notag\\
& \quad   + 2N^4\left(\int_{n}^{\infty}\alpha_t^2 dt\right)^2\sup_{t\geq0, j\geq1}\mE_{0}\left[\left(4+\alpha_n^2|\psi_{1,t}|^2\right) \frac{|g_{j,t}(c_t)|^4}{(1+|c_t|)^4}\varphi_t^2\right]\left(1+be^{\lambda (n+1)}\right)^4 \notag\\
&\quad   +\alpha_n^2 \sup_{t\geq0}\mE_{0}\left[|\psi_{2,t}|^2\right]b^2e^{2\lambda (n+1)}
 + \alpha_n^2 \sup_{t\geq0} \mE_{0}\left[|\psi_{1,t}|^2\right]b^4e^{4\lambda (n+1)} \notag\\
&\quad  +4\left(\int_{n}^{\infty}\alpha_t^2dt\right)^2\sup_{t\geq0}\mE_{0}\left[\left|\psi_{1,t}^T(A+\Delta_{1,t})\right|^2\right] b^4e^{4\lambda (n+1)} \notag\\
&\quad  +\left(\int_{n}^{\infty}\alpha_t^2dt\right)^2
            \sup_{t\geq0}  \mE_{0}\left[\left|2\psi_{1,t}^T\Delta_{2,t}+(A+\Delta_{1,t})^T\psi_{2,t}\right|^2\right]b^2e^{2\lambda (n+1)} \notag\\
&\quad +\left(\int_{n}^{\infty}\alpha_t^2dt\right)^2\sup_{t\geq0} \mE_{0}\left[\left|\psi_{2,t}^T\Delta_{2,t}\right|^2\right]\notag\\
&\quad +N^2\int_{n}^{\infty}\alpha_t^2 dt\sup_{t\geq0, j\geq1}\mE_{0}\left[(1-2\alpha_t\psi_{1,t}^T)^2\frac{|g_{j,t}(c_t)|^2}{(1+|c_t|)^2} \varphi_t\right]\left(1+be^{\lambda(n+1)}\right)^2b^2e^{2\lambda (n+1)}  \notag\\
&\quad +N^2\int_{n}^{\infty}\alpha_t^4dt\sup_{t\geq0, j\geq1}\mE_{0}\left[|\psi_{2,t}|^2\frac{|g_{j,t}(c_t)|^2}{(1+|c_t|)^2}\varphi_t\right] \left(1+be^{\lambda(n+1)}\right)^2\notag\\
&\quad +N^4\left(\int_{n}^{\infty}\alpha_t^2dt\right)^2 \mE_0\left[\varphi_t^2\right]b^4e^{4\lambda (n+1)}
        +\frac{N^2}{4}\left(\int_{\tau_n}^{\sigma_{n+1}}\alpha_t^2dt\right)^2\mE_0\left[\varphi_t^2\right].\label{proof:lemma-sa}
\end{align}
Hence, we can deduce from \eqref{proof:lemma-sa} that there exists $C_0>0$, so that 
\begin{align}
    \mE_0^{E_{n+1,b}}\left[\left|\left|c_{\sigma_{n+1}}\right|^2-\left|c_{\tau_n}\right|^{2}\right|^2\right]\leq C_0A_{n}b^4e^{4\lambda(n+1)}, \label{equ:proof-td-equ2}
\end{align}
where 
\begin{align*}
A_{n} = \alpha_n^2 + \left(\int_n^{\infty}\alpha_t^2dt\right)^2+ \int_n^{\infty}\alpha_t^2dt.
\end{align*}

As a result, we  have from \eqref{equ:proof-td-equ1} and \eqref{equ:proof-td-equ2}  that 
\begin{align*}
    \mathbb{P}\left({E_{n+1,b}}\right) \leq C_0A_ne^{4\lambda},\quad n\geq1.
\end{align*}
Thus, we have 
\begin{align*}
 \mathbb{P}\left(\overline E_{n+1}\right) 
 \geq\mathbb{P}(B_b\setminus E_{n+1,b})
 >1-\frac{\log\left(\delta^2+1\right)}{\log(b^2+1)}-C_0A_ne^{4\lambda}.
\end{align*}
In particular, this implies that $\mathbb{P}\left(\sup_{t\geq0}|c_t|<\infty\right)=1$.

Now we can write 
\begin{align*}
    \mE_0^{\overline E_{n+1}}\left[\left|c_{t}\right|^{2}\right]\leq
    \left|c_{0}\right|^{2} -2 \gamma \int_{0}^{t} \alpha_{s} \mE_0^{\overline E_{n+1}}\left[\left|c_{s}\right|^{2}\right] ds + \sqrt{C_0A_n}b^2e^{2\lambda n}.
\end{align*}
By the comparison lemma \citep[Lemma C.3.1]{Sontag1998}, we have 
\begin{align*}
    \mE_0^{\overline E_{n+1}}\left[\left|c_{t}\right|^{2}\right]\leq \left(\left|c_{0}\right|^2+\sqrt{C_0A_n}b^2e^{2\lambda n}\right)e^{-2 \gamma \int_{0}^{t} \alpha_{s}ds}.
\end{align*}
By Markov's inequality, one has for all $\epsilon>1$ that
\begin{align*}
    \mathbb{P}\left(\omega\in \overline E_{n+1}:\left|c_{t}(\omega)\right|^{2}\geq\epsilon\left(\left|c_{0}\right|^2+\sqrt{C_0A_n}b^2e^{2\lambda(n+1)}\right)e^{-2 \gamma \int_{0}^{t} \alpha_{s}ds}\right)\leq\frac{1}{\epsilon}.
\end{align*}
Hence, 
\begin{align*}
&  \quad  \mathbb{P}\left(\omega\in \Omega:\left|c_{t}(\omega)\right|^{2}<\epsilon\left(\left|c_{0}\right|^2+\sqrt{C_0A_n}b^2e^{2\lambda(n+1)}\right)e^{-2 \gamma \int_{0}^{t} \alpha_{s}ds}\right)\\
 &   >\mathbb{P}\left(\omega\in \overline E_{n+1}:\left|c_{t}(\omega)\right|^{2}<\epsilon\left(\left|c_{0}\right|^2+\sqrt{C_0A_n}b^2e^{2\lambda(n+1)}\right)e^{-2 \gamma \int_{0}^{t} \alpha_{s}ds}\right)\\
    &>1-\frac{\log\left(\delta^2+1\right)}{\log(b^2+1)}-C_0A_ne^{4\lambda}-\frac{1}{\epsilon}.
\end{align*}
Denote $\epsilon^{-1} = \frac{\log\left(\delta^2+1\right)}{\log(b^2+1)}$.
Then, 
\begin{align*}
    \delta^2 = e^{\log(\delta^2+1)}-1 = e^{{\epsilon^{-1}\log\left(b^2+1\right)}}-1 =\left(b^2+1\right)^{1/\epsilon}-1.
    \end{align*}
This completes the proof.
\end{proof}

\begin{lemma}\label{lem:bounds}
For any $0\leq s<t$, $a>0$ and $b>0$, we have
\begin{align*}
at-\sqrt{t+b}&\geq as-\sqrt{s+b}-ab-\frac{1}{4a}. 
\end{align*}
\end{lemma}
\begin{proof}
Denote $y_t = at-\sqrt{t+b}$.
Then
\begin{align*}
\dot y_t = a - \frac{1}{2\sqrt{t+b}}.
\end{align*}
If $a>(2\sqrt{b})^{-1}$, then $\dot y_t >0$ for all $t$. 
As a result, $y_t>y_s$ for all $t>s$.
If $a\leq(2\sqrt{b})^{-1}$, then $y_t-y_s \geq \min_ty_t -y_0$, where $\min_ty_t = -ab-(4a)^{-1}$.
This completes the proof.
\end{proof}

\begin{lemma}\label{lem:int-bound}
    Suppose $a$ is a positive real.
    Then for any $t_0>0$, 
    \begin{align*}
        \int_{t_0}^t\frac{1}{s}e^{as}ds\leq \frac{h}{t}e^{at}
    \end{align*}
    for some $h>0$.
\end{lemma}
\begin{proof}
    First, consider the following auxiliary dynamical system defined on $[t_0,\infty)$:
\begin{align*}
    \dot x_t & = z_t\quad x_{t_0}=0,\\
    \dot z_t & = -\frac{1}{t}z_t+az_t,\quad z_{t_0}=\frac{1}{t_0}e^{at_0}.
\end{align*}
To complete the proof, we only need to show $x_t<hz_t$ for some $h>0$.
Denote $e = x-hz$, where $h>0$ is a sufficiently large constant that will be determined later.
Then $e_{t_0}<0$, and we have
\begin{align*}
    \dot e & = \dot x-h\dot z = z_t+\frac{h}{t}z_t-ahz_t= z_t\left(1+\frac{h}{t}-ah\right).
\end{align*}
Now, if $a>t_0^{-1}$, we can choose $h>(a-t_0^{-1})^{-1}$, such that $1+ht^{-1}-ah<0$ for all $t\geq t_0$.
Since $z_t>0$, we have $\dot e_t<0$, and as a result $x_t< hz_t$ for all $t\geq t_0$.

If $a\leq t_0^{-1}$, then $1+ht^{-1}-ah=0$ admits a solution $t=t'\triangleq\frac{h}{ah-1}>t_0$
for all $h>a^{-1}$.
Hence, $e_t$ reaches its maximum at $t=t'$.
As a result, we can choose 
\begin{align*}
h\geq \frac{1}{ae}\int_{t_0}^{\frac{2}{a}}\frac{1}{s}e^{as}ds+\frac{2}{a},
\end{align*}
which implies that
\begin{align*}
\int_{t_0}^{\frac{h}{ha-1}}\frac{1}{s}e^{as}ds\leq (ha-1)e^{\frac{1}{ha-1}+1},
\end{align*}
and hence $e_{t_0}\leq0$.
This completes the proof.
\end{proof}

\begin{lemma}\label{lem:step}
Consider $\alpha_t>0$  with $\dot\alpha_t\leq0$.
If there exist $h>0$ and $t_0\geq0$, such that $\alpha_t^2\geq -h\dot \alpha_t$ for all $t\geq t_0$, then
\begin{align*}
e^{-\int_{t_0}^t\alpha_sds}\leq \left(\frac{\alpha_t}{\alpha_{t_0}}\right)^h.
\end{align*}
\end{lemma}
\begin{proof}
By definition,
\begin{align*}
 \alpha_t \geq -h\frac{\dot \alpha_t}{\alpha_t}\geq0,\quad \forall t\geq t_0.
\end{align*}
Taking integration from $t_0$ to $t$, we have
\begin{align*}
\int_{t_0}^t\alpha_sds \geq h\log(\alpha_{t_0}) -h\log(\alpha_t).
\end{align*}
Hence,
\begin{align*}
e^{-\int_{t_0}^t\alpha_sds}\leq \left(\frac{\alpha_t}{\alpha_{t_0}}\right)^h.
\end{align*}
This completes the proof.
\end{proof}

\vskip 0.2in

\bibliography{VIbib}

\end{document}